\documentclass[12pt]{elsarticle}
\usepackage{graphicx}
\usepackage{pgfplots}
\usepackage{amsfonts}
\usepackage{amsmath}
\usepackage{bm}
\usepackage[utf8]{inputenc}
\usepackage[english]{babel}
\usepackage{amsthm}
\usepackage{enumitem}
\usepackage{hyperref}
\usepackage{cancel}
\usepackage[linesnumbered]{algorithm2e}
\usepackage[utf8]{inputenc}
\usepackage[english]{babel}
 
\usepackage{natbib}
\setcitestyle{authoryear,open={(},close={)}}

\newcommand\backleadsto{\mathrel{\reflectbox{$\leadsto$}}}

\makeatletter
\newtheorem*{rep@theorem}{\rep@title}
\newcommand{\newreptheorem}[2]{%
\newenvironment{rep#1}[1]{%
 \def\rep@title{#2 \ref{##1}}%
 \begin{rep@theorem}}%
 {\end{rep@theorem}}}
\makeatother

\usepackage{subcaption} 
\usepackage{changepage}
\usepackage{xcolor,colortbl}
\usepackage{multicol}% 
\usepackage{tikz}
\usetikzlibrary{decorations.pathreplacing,angles,quotes}
\usepackage{tkz-tab}
\usetikzlibrary{plotmarks}

\usetikzlibrary{arrows, automata}
\usetikzlibrary{decorations.pathreplacing,angles,quotes}

\newcommand\ci{\perp\!\!\!\perp}

\def\circarrow{{\circ\hspace{0.3mm}\!\!\! \rightarrow}}
\def\circlinecirc{{\circ \hspace{0.4mm}\!\!\! - \hspace{0.4mm}\!\!\!\circ}}
\def\circline{{\circ \! -}}
\def\linecirc{{- \! \circ}}

\newtheorem{definition}{Definition}
\newtheorem{theorem}{Theorem}
\newtheorem{lemma}{Lemma}
\newtheorem{proposition}{Proposition}
\newtheorem{corollary}{Corollary}

\newreptheorem{definition}{Definition}
\newreptheorem{theorem}{Theorem}
\newreptheorem{lemma}{Lemma}
\newreptheorem{proposition}{Proposition}
\newreptheorem{corollary}{Corollary}
\newreptheorem{claim}{Claim}

\makeatletter
\def\ps@pprintTitle{%
 \let\@oddhead\@empty
 \let\@evenhead\@empty
 \def\@oddfoot{}%
 \let\@evenfoot\@oddfoot}
\makeatother

\begin{document}

\title{A Constraint-Based Algorithm For Causal Discovery\\ with Cycles, Latent Variables \& Selection Bias}

\author{ 
     Eric V. Strobl
}

\address{
    Pittsburgh, PA
}

\begin{abstract}
Causal processes in nature may contain cycles, and real datasets may violate causal sufficiency as well as contain selection bias. No constraint-based causal discovery algorithm can currently handle cycles, latent variables and selection bias (CLS) simultaneously. I therefore introduce an algorithm called Cyclic Causal Inference (CCI) that makes sound inferences with a conditional independence oracle under CLS, provided that we can represent the cyclic causal process as a non-recursive linear structural equation model with independent errors. Empirical results show that CCI outperforms CCD in the cyclic case as well as rivals FCI and RFCI in the acyclic case.
\end{abstract}

\maketitle

\section{The Problem}

Scientists often infer causation using data collected from randomized controlled experiments. However, randomized experiments can be slow, non-generalizable, unethical or expensive. Consider for example trying to discover the causes of a human illness, a common scenario in modern medical science. Performing interventions with possibly harmful consequences on humans is unethical, so scientists often perform experiments on animals instead knowing full well that the causal relationships discovered in animals may not generalize to humans. Moreover, many possible causes for an illness often exist, so scientists typically perform numerous animal experiments in order to discover the causes. A lengthy trial and error process therefore ensues at considerable financial expense. We would ideally like to speed up this scientific process by discovering causation directly from human observational data, which we can more easily acquire.

The need for faster causal discovery has motivated many to develop algorithms for inferring causation from observational data. The PC algorithm, for example, represents one the earliest algorithms for inferring causation using i.i.d. data collected from an underlying acyclic causal process \citep{Spirtes00}. PC actually falls within a wider class of causal discovery algorithms called constraint-based (CB) algorithms which utilize a conditional independence (CI) oracle, or a CI test in the finite sample case, to re-construct the underlying causal graph. The FCI algorithm is another example of a CB algorithm which extends PC to handle latent variables and selection bias \citep{Spirtes00,Zhang08}. Yet another CB algorithm called CCD cannot handle latent variables (and perhaps selection bias) like FCI, but CCD can infer cyclic causal structure provided that all causal relations are linear \citep{Richardson96, Richardson99}. Many other CB algorithms exist, and most of these methods come with some guarantee of soundness in the sense that their outputs are provably correct with a CI oracle.

The aforementioned CB algorithms and many other non-CB algorithms have been successful in inferring causation under their respective assumptions. However, causal processes and datasets encountered in practice may not satisfy the assumptions of the algorithms. In particular, many causal processes are known to contain feedback loops \citep{Sachs05}, and datasets may contain latent variables as well as some degree of selection bias \citep{Spirtes95}. Few algorithms can handle cycles, latent variables and selection bias (CLS) simultaneously, so scientists often must unwillingly apply other methods knowing that the outputs may introduce unwanted bias \citep{Sachs05,Mooij13}. Solving the problem of causal discovery under CLS would therefore provide a much needed basis for justifying the output of causal discovery algorithms when run on real data.

A few investigators have devised non-CB based solutions for the problem of causal discovery  under CLS. \cite{Hyttinen13} introduced the first approach, where CI constraints are fed into a SAT solver which then outputs a graph consistent with the constraints. However, the method can be slow because the SAT solver does not construct efficient test schedules like CB algorithms. \cite{Strobl_thesis17} provided a different solution in the Gaussian case, provided that the cyclic causal process can be decomposed into a set of acyclic ones. The method uses both conditional independence testing and mixture modeling, but the mixture modeling inhibits a straightforward extension of the method to the non-parametric setting even in the linear case. Existing solutions to the problem of causal discovery under CLS thus fall short in either efficiency or generalizability.

The purpose of this paper is to introduce the first CB algorithm that is sound under CLS. The method is efficient because it constructs small test schedules, and it is generalizable to the non-parametric setting because the algorithm only requires a sound CI test. We introduce the new CB algorithm as follows. We first provide provide background material on causal discovery without cycles in Sections \ref{sec_graph} through \ref{sec_MAGs}. We then review causal discovery with cycles in Section \ref{sec_DCGs}. Section \ref{sec_MAAGs} introduces the new notion of a maximal almost ancestral graph (MAAG) for summarizing cyclic graphs with latent variables and selection bias. Next, Section \ref{sec_overview} contains an overview of the proposed algorithm, while Section \ref{sec_trace} outlines an algorithm trace. Section \ref{sec_details} subsequently lists the details of the proposed CB algorithm. Experimental results are included in Section \ref{sec_exp}. We finally conclude the paper in Section \ref{sec_conc}. Most of the proofs are located in the Appendix.

\section{Graph Terminology} \label{sec_graph}

Let italicized capital letters such as $A$ denote a single variable and bolded as well as italicized capital letters such as $\bm{A}$ denote a set of variables (unless specified otherwise). We will also use the terms ``variables'' and ``vertices'' interchangeably.

A graph $\mathbb{G}=(\bm{X}, \mathcal{E})$ consists of a set of vertices $\bm{X}=\{ X_1, \dots, X_p \}$ and a set of edges $\mathcal{E}$ between each pair of vertices. The edge set $\mathcal{E}$ may contain the following six edge types: $\rightarrow$ (directed), $\leftrightarrow$ (bidirected), --- (undirected), $\circarrow$ (partially directed), $\circline$ (partially undirected) and $\circlinecirc$ (nondirected). Notice that these six edges utilize three types of endpoints including \textit{tails}, \textit{arrowheads}, and \textit{circles}.

We call a graph containing only directed edges as a \textit{directed graph}. We will only consider directed graphs without self-loops in this paper. On the other hand, a \textit{mixed graph} contains directed, bidirected and undirected edges. We say that $X_i$ and $X_j$ are \textit{adjacent} in a graph, if they are connected by an edge independent of the edge's type. An \textit{(undirected) path} $\Pi$ between $X_i$ and $X_j$ is a set of consecutive edges (also independent of their type) connecting the variables such that no vertex is visited more than once. A \textit{directed path} from $X_i$ to $X_j$ is a set of consecutive directed edges from $X_i$ to $X_j$ in the direction of the arrowheads. A \textit{cycle} occurs when a path exists from $X_i$ to $X_j$, and $X_j$ and $X_i$ are adjacent. More specifically, a directed path from $X_i$ to $X_j$ forms a \textit{directed cycle} with the directed edge $X_j \rightarrow X_i$ and an \textit{almost directed cycle} with the bidirected edge $X_j \leftrightarrow X_i$. We call a directed graph a \textit{directed acyclic graph} (DAG), if it does not contain directed cycles.

Three vertices $\{X_i,X_j,X_k\}$ form an \textit{unshielded triple}, if $X_i$ and $X_j$ are adjacent, $X_j$ and $X_k$ are adjacent, but $X_i$ and $X_k$ are not adjacent. On the other hand, the three vertices form a \textit{triangle} when $X_i$ and $X_k$ are also adjacent. We call a nonendpoint vertex $X_j$ on a path $\Pi$ a \textit{collider} on $\Pi$, if both the edges immediately preceding and succeeding the vertex have an arrowhead at $X_j$. Likewise, we refer to a nonendpoint vertex $X_j$ on $\Pi$ which is not a collider as a \textit{non-collider}. Finally, an unshielded triple involving $\{X_i,X_j,X_k\}$ is more specifically called a \textit{v-structure}, if $X_j$ is a collider on the subpath $\langle X_i,X_j,X_k \rangle$.

We say that $X_i$ is an \textit{ancestor} of $X_j$ (and $X_j$ is a \textit{descendant} of $X_i$) if and only if there exists a directed path from $X_i$ to $X_j$ or $X_i = X_j$. We write $X_i \in \textnormal{Anc}(X_j)$ to mean $X_i$ is an ancestor of $X_j$ and $X_j \in \textnormal{Des}(X_i)$ to mean $X_j$ is a descendant of $X_i$. We also apply the definitions of an ancestor and descendant to a set of vertices $\bm{Y} \subseteq \bm{X}$ as follows: 
\begin{equation} \nonumber
\begin{aligned}
\textnormal{Anc}(\bm{Y}) &= \{X_i | X_i \in \textnormal{Anc}(X_j) \text{ for some } X_j \in \bm{Y}\},\\
\textnormal{Des}(\bm{Y}) &= \{X_i | X_i \in \textnormal{Des}(X_j) \text{ for some } X_j \in \bm{Y}\}.
\end{aligned}
\end{equation}

\section{Causal \& Probabilistic Interpretations of DAGs} \label{sec_DAGs}

We will interpret DAGs in a causal fashion \citep{Spirtes00,Pearl09}. To do this, we consider a stochastic causal process with a distribution $\mathbb{P}$ over $\bm{X}$ that satisfies the \textit{Markov property}. A distribution satisfies the Markov property if it admits a density that ``factorizes according to the DAG'' as follows:
\begin{equation} \label{eq_fac}
f(\bm{X})=\prod_{i=1}^{p} f(X_i | \textnormal{Pa}(X_i)).
\end{equation}
\noindent We can in turn relate the above equation to a graphical criterion called d-connection. Specifically, if $\mathbb{G}$ is a directed graph in which $\bm{A}$, $\bm{B}$ and $\bm{C}$ are disjoint sets of vertices in $\bm{X}$, then $\bm{A}$ and $\bm{B}$ are \textit{d-connected} by $\bm{C}$ in the directed graph $\mathbb{G}$ if and only if there exists an \textit{active or d-connecting path} $\Pi$ between some vertex in $\bm{A}$ and some vertex in $\bm{B}$ given $\bm{C}$. An active path between $\bm{A}$ and $\bm{B}$ given $\bm{C}$ refers to an undirected path $\Pi$ between some vertex in $\bm{A}$ and some vertex in $\bm{B}$ such that, for any collider $X_i$ on $\Pi$, a descendant of $X_i$ is in $\bm{C}$ and no non-collider on $\Pi$ is in $\bm{C}$. A path is \textit{inactive} when it is not active. Now $\bm{A}$ and $\bm{B}$ are \textit{d-separated} by $\bm{C}$ in $\mathbb{G}$ if and only if they are not d-connected by $\bm{C}$ in $\mathbb{G}$. For shorthand, we will write $\bm{A} \ci_d \bm{B} | \bm{C}$ and $\bm{A} \not \ci_d \bm{B} | \bm{C}$ when $\bm{A}$ and $\bm{B}$ are d-separated or d-connected given $\bm{C}$, respectively. The conditioning set $\bm{C}$ is called a \textit{minimal separating set} if and only if $\bm{A} \ci_d \bm{B} | \bm{C}$ but $\bm{A}$ and $\bm{B}$ are d-connected given any proper subset of $\bm{C}$. 

If we have $\bm{A} \ci_d \bm{B} | \bm{C}$, then $\bm{A}$ and $\bm{B}$ are conditionally independent given $\bm{C}$, denoted as $\bm{A} \ci \bm{B} | \bm{C}$, in any joint density factorizing according to \eqref{eq_fac} \citep{Lauritzen90}; we refer to this property as the \textit{global directed Markov property}. We also refer to the converse of the global directed Markov property as \textit{d-separation faithfulness}; that is, if $\bm{A} \ci \bm{B}|\bm{C}$, then $\bm{A}$ and $\bm{B}$ are d-separated given $\bm{C}$. One can in fact show that the factorization in \eqref{eq_fac} and the global directed Markov property are equivalent, so long as the distribution over $\bm{X}$ admits a density \citep{Lauritzen90}. We will only consider distributions which admit densities in this report, so we will use the terms ``distribution'' and ``density'' interchangeably from here on out.

\section{Ancestral Graphs for DAGs} \label{sec_MAGs}

We can associate a directed graph $\mathbb{G}$ with a mixed graph $\mathbb{G}^\prime$ with arbitrary edges as follows. For any directed graph $\mathbb{G}$ with vertices $\bm{X}$, we consider the partition $\bm{X} = \bm{O} \cup \bm{L} \cup \bm{S}$, where $\bm{O}$, $\bm{L}$ and $\bm{S}$ are non-overlapping sets of observable, latent and selection variables, respectively. We then consider a mixed graph $\mathbb{G}^\prime$ over $\bm{O}$, where the arrowheads and tails have the following interpretations. If we have the arrowhead $O_i * \!\! \rightarrow O_j$, where the asterisk is a meta-symbol denoting either a tail or an arrowhead, then we say that $O_j$ is not an ancestor of $O_i \cup \bm{S}$ in $\mathbb{G}$. On the other hand, if we have the tail $O_i * \!\! - O_j$ then we say that $O_j$ is an ancestor of $O_i \cup \bm{S}$ in $\mathbb{G}$. Let $\textnormal{Anc}(X_i)$ denote the ancestors of $X_i$ in $\mathbb{G}$. Obviously then, any $\mathbb{G}^\prime$ constructed from a directed graph cannot have a \textit{directed cycle}, where $O_i \rightarrow O_j$ in $\mathbb{G}^\prime$ and $O_j \in \textnormal{Anc}(O_i \cup \bm{S})$ in $\mathbb{G}$. Similarly, $\mathbb{G}^\prime$ cannot have an \textit{almost directed cycle}, where $O_i \leftrightarrow O_j$ is in $\mathbb{G}^\prime$ and $O_j \in \textnormal{Anc}(O_i \cup \bm{S})$ in $\mathbb{G}$.

One can also show that, if $\mathbb{G}$ is acyclic, then any mixed graph constructed from $\mathbb{G}$ cannot have a undirected edge $O_i - O_j$ with incoming arrowheads at $O_i$ or $O_j$ \citep{Richardson00}. We therefore find it useful to consider a subclass of mixed graphs called \textit{ancestral graphs}:
\begin{definition}\label{def_AG}(Ancestral Graphs) A mixed graph $\mathbb{G}^\prime$ is more specifically called an ancestral graph if and only if $\mathbb{G}^\prime$ satisfies the following three properties:
\begin{enumerate} \label{AG_prop}
\item There is no directed cycle.
\item There is no almost directed cycle.
\item For any undirected edge $O_i - O_j$, $O_i$ and $O_j$ have no incoming arrowheads.
\end{enumerate}
\end{definition}
\noindent Observe that every mixed graph of a DAG is an ancestral graph.

A \textit{maximal ancestral graph} (MAG) is an ancestral graph where every missing edge corresponds to a conditional independence relation. One can transform a DAG $\mathbb{G}$ into a MAG $\mathbb{G}^\prime$ as follows. First, for any pair of vertices $\{O_i, O_j\}$, make them adjacent in $\mathbb{G}^\prime$ if and only if there is an \textit{inducing path} between $O_i$ and $O_j$ in $\mathbb{G}$. We define an inducing path as follows:
\begin{definition} \label{def_IP}
(Inducing Path) A path $\Pi$ between $O_i$ and $O_j$ in $\mathbb{G}$ is called an inducing path if and only if every collider on $\Pi$ is an ancestor of $\{O_i,O_j\} \cup \bm{S}$, and every non-collider on $\Pi$ (except for the endpoints) is in $\bm{L}$.
\end{definition}
\noindent Note that two observables $O_i$ and $O_j$ are connected by an inducing path if and only if they are d-connected given any $\bm{W} \subseteq \bm{O} \setminus \{ O_i, O_j \}$ as well as $\bm{S}$ \citep{Spirtes00}. Then, for each adjacency $O_i *\!\! -\!\! * O_j$ in $\mathbb{G}^\prime$, we have the following edge interpretations:
\begin{enumerate}
\item If we have $O_i * \!\! \rightarrow O_j$, then $O_j \not \in \textnormal{Anc}(O_i \cup \bm{S})$ in $\mathbb{G}$.
\item If we have $O_i * \!\! \textnormal{---} O_j$, then $O_j \in \textnormal{Anc}(O_i \cup \bm{S})$ in $\mathbb{G}$.
\end{enumerate}
The MAG of a DAG is therefore a kind of marginal graph that does not contain the latent or selection variables, but does contain information about the ancestral relations between the observable and selection variables in the DAG. The MAG also has the same d-separation relations as the DAG, specifically among the observable variables conditional on the selection variables \citep{Spirtes96}. 

\section{Directed Cyclic Graphs as Equilibriated Causal Processes} \label{sec_DCGs}
We now allow cycles in a directed graph. Multiple different causal representations of a directed cyclic graph exist in the literature. Examples include dynamic Bayesian networks \citep{Dagum95}, structural equation models with feedback \citep{Spirtes95c}, chain graphs \citep{Lauritzen02} and mixtures of DAGs \citep{Strobl_thesis17}. See \citep{Strobl_thesis17} for a discussion of each of their strengths and weaknesses. In this report, we will only consider structural equation models with feedback. 

Recall that the density $f(\bm{X})$ associated with a DAG $\mathbb{G}$ obeys the global Markov property. However, the density may not obey the global Markov property if $\mathbb{G}$ contains cycles. We therefore must impose certain assumptions on $\mathbb{P}$ such that its density does obey the property.

\cite{Spirtes95c} proposed the following assumptions on $\mathbb{P}$. We say that a distribution $\mathbb{P}$ obeys a \textit{structural equation model with independent errors} (SEM-IE) with respect to $\mathbb{G}$ if and only if we can describe $\bm{X}$ as $X_i = g_i(\textnormal{Pa}(X_i), \varepsilon_i)$ for all $X_i \in \bm{X}$ such that $X_i$ is $\sigma(\textnormal{Pa}(X_i), \varepsilon_i)$ measurable and $\varepsilon_i \in \bm{\varepsilon}$ \citep{Evans16}. Here, we have a set of jointly independent errors $\bm{\varepsilon}$, and $\sigma(Y)$ refers to the sigma-algebra generated by the random variable $Y$. An example of an SEM-IE is illustrated below with an associated directed graph drawn in Figure \ref{fig_SEMIE}:
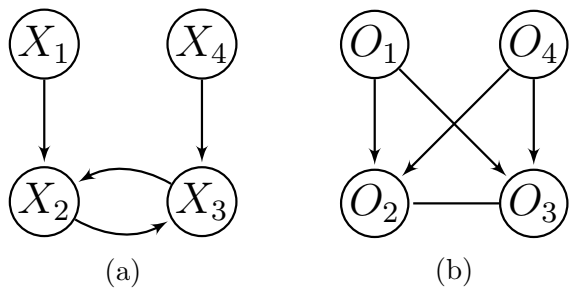
\begin{figure}
\centering
\begin{subfigure}{.25\linewidth}
\centering
\resizebox{\linewidth}{!}{
\begin{tikzpicture}[scale=1.0, shorten >=1pt,auto,node distance=2.8cm, semithick]
                    
\tikzset{vertex/.style = {shape=circle,draw,inner sep=0.4pt}}
 
\node[vertex] (1) at  (0,1) {$X_1$};
\node[vertex] (4) at  (1.5,1) {$X_4$};
\node[vertex] (3) at  (1.5,-0.5) {$X_3$};
\node[vertex] (2) at  (0,-0.5) {$X_2$};

\draw[->,> = latex'] (1) to (2);
\draw[->,> = latex', bend right] (2) to (3);
\draw[->,> = latex', bend right] (3) to (2);
\draw[->,> = latex'] (4) to (3);
\end{tikzpicture}
} 
\caption{} \label{fig_SEMIE}
\end{subfigure} \hspace{7mm}
\begin{subfigure}{.25\linewidth}
\centering
\resizebox{\linewidth}{!}{
\begin{tikzpicture}[scale=1.0, shorten >=1pt,auto,node distance=2.8cm, semithick]
                    
\tikzset{vertex/.style = {shape=circle,draw,inner sep=0.4pt}}
 
\node[vertex] (1) at  (0,1) {$O_1$};
\node[vertex] (4) at  (1.5,1) {$O_4$};
\node[vertex] (3) at  (1.5,-0.5) {$O_3$};
\node[vertex] (2) at  (0,-0.5) {$O_2$};

\draw[->,> = latex'] (1) to (2);
\draw[-] (3) to (2);
\draw[->,> = latex'] (4) to (2);
\draw[->,> = latex'] (4) to (3);
\draw[->,> = latex'] (1) to (3);
\end{tikzpicture}
}
\caption{} \label{fig_SEMIE_MAAG}
\end{subfigure}
\caption{(a) The graph $\mathbb{G}$ associated with the SEM-IE in Equation 
\eqref{struc_eqs}. (b) The associated MAAG.} 
\end{figure}

\begin{equation} \label{struc_eqs}
\begin{aligned}
X_1 &= \varepsilon_1, \\
X_2 &= B_{21}X_1 + B_{23}X_3 + \varepsilon_2, \\
X_3 &= B_{34}X_4 + B_{32}X_2 + \varepsilon_3, \\
X_4 &= \varepsilon_4,
\end{aligned}
\end{equation}

\noindent where $\bm{\varepsilon}$ denotes a set of jointly independent standard Gaussian error terms, and $B$ is a 4 by 4 coefficient matrix. Notice that the structural equations in \eqref{struc_eqs} are linear structural equations.

We can simulate data from an SEM-IE using the \textit{fixed point method} \citep{Fisher70}. The fixed point method involves two steps per sample. We first sample the error terms according to their independent distributions and initialize $\bm{X}$ to some values. Next, we apply the structural equations iteratively until the values of the random variables converge to values which satisfy the structural equations; in other words, the values converge almost surely to a fixed point.\footnote{We can perform the fixed point method more efficiently in the linear case by first representing the structural equations in matrix format: $\bm{X} = B \bm{X} + \bm{\varepsilon}$. Then, after drawing the values of $\bm{\varepsilon}$, we can obtain the values of $\bm{X}$ by solving the following system of equations: $\bm{X} = (\mathbb{I}-B)^{-1}\bm{\varepsilon}$, where $\mathbb{I}$ denotes the identity matrix.} Note that the values of the random variables may not necessarily converge to a fixed point all of the time for every set of structural equations and error distributions, but we will only consider those structural equations and error distributions which do satisfy this property. We call the distribution reached at the fixed points as the \textit{equilibrium distribution}.

\cite{Spirtes95c} proved the following regarding \textit{linear SEM-IEs}, or SEM-IEs with linear structural equations:
\begin{theorem} \label{thm_linearSEMIE}
The equilibrium distribution $\mathbb{P}$ of a linear SEM-IE satisfies the global directed Markov property with respect to the SEM-IE's directed graph $\mathbb{G}$ (acyclic or cyclic).
\end{theorem}
\noindent The above theorem provided a basis from which Richardson started constructing the Cyclic Causal Discovery (CCD) algorithm \citep{Richardson96,Richardson99} for causal discovery with feedback.

\section{Almost Ancestral Graphs for Directed Graphs} \label{sec_MAAGs}
Recall that an ancestral graph satisfies the three properties listed in Definition \ref{def_AG}. We now define an \textit{almost ancestral graph} (AAG) which only satisfies the first two conditions of an ancestral graph. The following result should be obvious:
\begin{proposition}
Any mixed graph $\mathbb{G}^\prime$ constructed from a directed graph $\mathbb{G}$ (cyclic or acyclic) over $\bm{O}$ is an AAG.
\end{proposition}
\begin{proof}
No directed cycle and no almost directed cycle can exist in $\mathbb{G}^\prime$ because that would imply that there exists a vertex $O_j$ which is simultaneously both an ancestor of $O_i \cup \bm{S}$ and not an ancestor of $O_i \cup \bm{S}$ in $\mathbb{G}$.
\end{proof}

Now an almost ancestral graph is said to \textit{maximal} when an edge exists between any two vertices $O_i$ and $O_j$ if and only if there exists an inducing path between $O_i$ and $O_j$. Note that a maximal almost ancestral graph (MAAG) $\mathbb{G}^\prime$ does not necessarily preserve the d-separation relations between the variables in $\bm{O}$ given $\bm{S}$ in a directed graph $\mathbb{G}$, even though $\mathbb{G}^\prime$ does do so when $\mathbb{G}$ is acyclic \citep{Spirtes96}. We provide an example of an MAAG in Figure \ref{fig_SEMIE_MAAG}, where $\bm{X} = \bm{O}$ because $\bm{L} = \emptyset$ and $\bm{S} = \emptyset$.

\section{Overview of the Proposed Algorithm} \label{sec_overview}

We now introduce a new CB algorithm called Cyclic Causal Inference (CCI) for discovering causal relations under CLS. We first provide a bird's eye view of the algorithm (for more details, see Section \ref{sec_details}). We will assume that the reader is familiar with past CB algorithms including PC, FCI, RFCI and CCD; for a brief review of each, see Section \ref{sec_algs} in the Appendix.

We have summarized CCI in Algorithm 1. The algorithm first performs FCI's skeleton discovery procedure in Step \ref{alg_skeleton} (see Algorithm \ref{fci_final_skel}), which discovers a graph where each adjacency between any two observables $O_i$ and $O_j$ corresponds to an inducing path even in the cyclic case (see Lemma \ref{lem_pdsep2} of Section \ref{sec_proofs} for a proof of this statement). The algorithm then orients v-structures in Step \ref{alg_vstruc} using FCI's v-structure discovery procedure (Algorithm \ref{fci_vstruc}).

Step \ref{alg_addV} of CCI checks for additional long range d-separation relations. Recall that the PC algorithm only checks for short range d-separation relations via v-structures. However, two cyclic directed graphs may agree locally on d-separation relations, but disagree on d-separation relations between distant variables, even if they do not contain any latent variables or selection bias  \citep{Richardson94}. As a result, Step \ref{alg_addV} allows the algorithm to orient additional edges by checking for additional d-separation relations.

Step \ref{alg_addD} of CCI discovers non-minimal d-separating sets which were not discovered in Step \ref{alg_skeleton}. Recall that, if we have $O_i * \!\! \rightarrow O_j \leftarrow \!\!* O_k$ with $O_i$ and $O_k$ non-adjacent, then every set d-separating $O_i$ and $O_k$ does not contain $O_j$. However, in the cyclic case, $O_i$ and $O_k$ can be d-separated given a set that contains $O_j$. It turns out that we can infer additional properties about the MAAG, if we find d-separating sets which contain $O_j$. Step \ref{alg_addD} of Algorithm \ref{alg_CCI} therefore discovers these additional non-minimal d-separating sets using Algorithm \ref{alg_CCI_E}. Steps \ref{alg_step5} and \ref{alg_step6} in turn utilize the d-separating sets discovered in Steps \ref{alg_skeleton} and \ref{alg_addD} in order to orient additional edges. Finally, Step \ref{alg_OR} applies the 7 orientation rules described in Section \ref{sec_details}. CCI thus ultimately outputs a \textit{partially oriented MAAG}, or an MAAG with tails, arrowheads and unspecified endpoints denoted by circles.

\begin{algorithm}[]
 \KwData{CI oracle}
 \KwResult{$\widehat{\mathbb{G}}$}
 \BlankLine
 
Run FCI's skeleton discovery procedure (Algorithm \ref{fci_final_skel}).\\ \label{alg_skeleton}
Run FCI's v-structure orientation procedure (Algorithm \ref{fci_vstruc}). \\\label{alg_vstruc}

For any triple of vertices $\langle O_i, O_k, O_j \rangle$ such that we have $O_k \circline \!\! * O_i$, if there is a set in $\textnormal{Sep}(O_i,O_j)$ discovered in Step \ref{alg_skeleton} such that $O_k \not \in \textnormal{Sep}(O_i,O_j)$, $O_i \not \ci_d O_k | \textnormal{Sep}(O_i,O_j) \cup \bm{S}$ and $O_j \not \ci_d O_k | \textnormal{Sep}(O_i,O_j) \cup \bm{S}$, then orient $O_k \circline \!\! * O_i$ as $O_k \leftarrow \!\! * O_i$. \label{alg_addV}

Find additional non-minimal d-separating sets using Algorithm \ref{alg_CCI_E}.\\ \label{alg_addD}

Find all quadruples of vertices $\langle O_i, O_j, O_k, O_l \rangle$ such that $O_i$ and $O_k$ non-adjacent, $O_i * \!\! \rightarrow O_l \leftarrow \!\! * O_k$, and $O_i \ci_d O_k|\bm{W} \cup \bm{S}$ with $O_j \in \bm{W}$ and $\bm{W} \subseteq \bm{O} \setminus \{O_i, O_k\}$. If $O_l \not \in \bm{W} = \textnormal{Sep}(O_i,O_k)$ as discovered in Step \ref{alg_skeleton}, then orient $O_j * \!\! \linecirc O_l$ as $O_j * \!\! \rightarrow O_l$. If we also have $O_i * \!\! \rightarrow O_j \leftarrow \!\! * O_k$ and $O_l \in \bm{W} = \textnormal{SupSep}(O_i,O_j,O_k)$ as discovered in Step \ref{alg_addD}, then orient $O_j * \!\! \linecirc O_l$ as $O_j * \!\! - O_l$.  \label{alg_step5}

If we have $O_i \ci_d O_k|\bm{W} \cup \bm{S}$ for some $\bm{W} \subseteq \bm{O} \setminus \{O_i, O_k\}$ with $O_j \in \bm{W}$ but we have $O_i \not \ci_d O_k|O_l \cup \bm{W} \cup \bm{S}$, then orient $O_l \circline \!\! * O_j$ as $O_l \leftarrow \!\! * O_j$. \\\label{alg_step6}

Execute orientation rules 1-7.\\ \label{alg_OR}
 
 \BlankLine

 \caption{Cyclic Causal Inference (CCI)} \label{alg_CCI}
\end{algorithm}

Now Algorithm \ref{alg_CCI} is sound due to the following theorem: 
\begin{theorem} \label{thm_main}
Consider a DAG or a linear SEM-IE with directed cyclic graph $\mathbb{G}$. If d-separation faithfulness holds, then CCI outputs a partially oriented MAAG of $\mathbb{G}$.
\end{theorem}

\section{Algorithm Trace} \label{sec_trace}

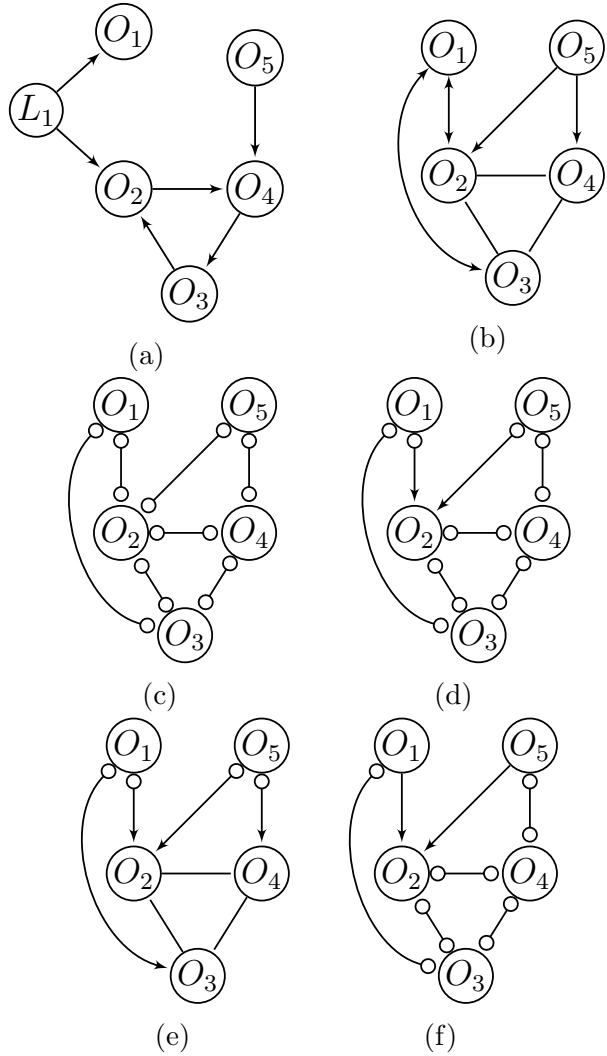
\begin{figure}
\centering
\begin{subfigure}{.29\linewidth}
\centering
\resizebox{\linewidth}{!}{
\begin{tikzpicture}[scale=1.0, shorten >=1pt,auto,node distance=2.8cm, semithick]
                    
\tikzset{vertex/.style = {shape=circle,draw,inner sep=0.4pt}}
 
\node[vertex] (1) at  (0,1) {$O_1$};
\node[vertex] (5) at  (1.5,0.7) {$O_5$};
\node[vertex] (2) at  (0,-0.8) {$O_2$};
\node[vertex] (3) at  (0.75,-2.0) {$O_3$};
\node[vertex] (4) at  (1.5,-0.8) {$O_4$};
\node[vertex] (6) at  (-1,0.1) {$L_1$};

\draw[->,> = latex'] (6) to (1);
\draw[->,> = latex'] (6) to (2);
\draw[->,> = latex'] (3) to (2);
\draw[->,> = latex'] (4) to (3);
\draw[->,> = latex'] (2) to (4);
\draw[->,> = latex'] (5) to (4);
\end{tikzpicture}
}
\caption{}  \label{fig_trace_a}
\end{subfigure}\hspace{7mm}
\begin{subfigure}{.25\linewidth}
\centering
\resizebox{\linewidth}{!}{
\begin{tikzpicture}[scale=1.0, shorten >=1pt,auto,node distance=2.8cm, semithick]
                    
\tikzset{vertex/.style = {shape=circle,draw,inner sep=0.4pt}}
 
\node[vertex] (1) at  (0,0.7) {$O_1$};
\node[vertex] (5) at  (1.5,0.7) {$O_5$};
\node[vertex] (2) at  (0,-0.8) {$O_2$};
\node[vertex] (3) at  (0.75,-2.0) {$O_3$};
\node[vertex] (4) at  (1.5,-0.8) {$O_4$};

\draw[<->,> = latex'] (1) to (2);
\draw[-] (3) to (2);
\draw[-] (4) to (3);
\draw[-] (2) to (4);
\draw[->,> = latex'] (5) to (4);
\draw[->,> = latex'] (5) to (2);
\draw[<->,> = latex', bend right = 60] (1) to (3);
\end{tikzpicture}
}
\caption{}  \label{fig_trace_b}
\end{subfigure}

\begin{subfigure}{.25\linewidth}
\centering
\resizebox{\linewidth}{!}{
\begin{tikzpicture}[scale=1.0, shorten >=1pt,auto,node distance=2.8cm, semithick]
                    
\tikzset{vertex/.style = {shape=circle,draw,inner sep=0.4pt}}
 
\node[vertex] (1) at  (0,0.7) {$O_1$};
\node[vertex] (5) at  (1.5,0.7) {$O_5$};
\node[vertex] (2) at  (0,-0.8) {$O_2$};
\node[vertex] (3) at  (0.75,-2.0) {$O_3$};
\node[vertex] (4) at  (1.5,-0.8) {$O_4$};

\draw[o-o] (1) to (2);
\draw[o-o, bend right=60] (1) to (3);
\draw[o-o] (5) to (4);
\draw[o-o] (5) to (2);
\draw[o-o] (3) to (2);
\draw[o-o] (4) to (3);
\draw[o-o] (2) to (4);
\end{tikzpicture}
}
\caption{} \label{fig_trace_c}
\end{subfigure} \hspace{2mm}
\begin{subfigure}{.25\linewidth}
\centering
\resizebox{\linewidth}{!}{
\begin{tikzpicture}[scale=1.0, shorten >=1pt,auto,node distance=2.8cm, semithick]
                    
\tikzset{vertex/.style = {shape=circle,draw,inner sep=0.4pt}}
 
\node[vertex] (1) at  (0,0.7) {$O_1$};
\node[vertex] (5) at  (1.5,0.7) {$O_5$};
\node[vertex] (2) at  (0,-0.8) {$O_2$};
\node[vertex] (3) at  (0.75,-2.0) {$O_3$};
\node[vertex] (4) at  (1.5,-0.8) {$O_4$};

\draw[o-o, bend right=60] (1) to (3);
\draw[o-o] (5) to (4);
\draw[o->,> = latex'] (5) to (2);
\draw[o->,> = latex'] (1) to (2);
\draw[o-o] (3) to (2);
\draw[o-o] (4) to (3);
\draw[o-o] (2) to (4);
\end{tikzpicture}
}
\caption{}  \label{fig_trace_d}
\end{subfigure}

\begin{subfigure}{.25\linewidth}
\centering
\resizebox{\linewidth}{!}{
\begin{tikzpicture}[scale=1.0, shorten >=1pt,auto,node distance=2.8cm, semithick]
                    
\tikzset{vertex/.style = {shape=circle,draw,inner sep=0.4pt}}
 
\node[vertex] (1) at  (0,0.7) {$O_1$};
\node[vertex] (5) at  (1.5,0.7) {$O_5$};
\node[vertex] (2) at  (0,-0.8) {$O_2$};
\node[vertex] (3) at  (0.75,-2.0) {$O_3$};
\node[vertex] (4) at  (1.5,-0.8) {$O_4$};

\draw[o->, > = latex', bend right=60] (1) to (3);
\draw[o->,> = latex'] (5) to (4);
\draw[o->,> = latex'] (5) to (2);
\draw[o->,> = latex'] (1) to (2);
\draw[-] (3) to (2);
\draw[-] (4) to (3);
\draw[-] (2) to (4);
\end{tikzpicture}
}
\caption{}  \label{fig_trace_f}
\end{subfigure}
\begin{subfigure}{.25\linewidth}
\centering
\resizebox{\linewidth}{!}{
\begin{tikzpicture}[scale=1.0, shorten >=1pt,auto,node distance=2.8cm, semithick]
                    
\tikzset{vertex/.style = {shape=circle,draw,inner sep=0.4pt}}
 
\node[vertex] (1) at  (0,0.7) {$O_1$};
\node[vertex] (5) at  (1.5,0.7) {$O_5$};
\node[vertex] (2) at  (0,-0.8) {$O_2$};
\node[vertex] (3) at  (0.75,-2.0) {$O_3$};
\node[vertex] (4) at  (1.5,-0.8) {$O_4$};

\draw[->,> = latex'] (1) to (2);
\draw[o-o, bend right=60] (1) to (3);
\draw[o-o] (5) to (4);
\draw[->,> = latex'] (5) to (2);
\draw[o-o] (3) to (2);
\draw[o-o] (4) to (3);
\draw[o-o] (2) to (4);
\end{tikzpicture}
}
\caption{}  \label{fig_trace_g}
\end{subfigure}
\caption{A sample run of CCI. The ground truth directed graph is illustrated in (a) with corresponding MAAG in (b). Step \ref{alg_skeleton} of CCI outputs (c), Step \ref{alg_vstruc} (d), and Step \ref{alg_OR} the final output (e). In contrast, CCD outputs (f). }
\end{figure}

We now illustrate a sample run of the CCI algorithm with a CI oracle. Consider the directed graph in Figure \ref{fig_trace_a} containing just one latent variable $L_1$ with MAAG in Figure \ref{fig_trace_b}. CCI proceeds as follows:
\begin{enumerate}[label=Step \arabic*:, leftmargin=1.5cm]
\item Discovers the skeleton shown in Figure \ref{fig_trace_c}.
\item Adds two arrowheads onto $O_2$ yielding Figure \ref{fig_trace_d} because $O_1 \ci_d O_5$.
\item Does not orient any endpoints.
\item Discovers the additional d-separating set $O_1 \ci_d O_5 | \{O_2,O_3\}$.
\item Does not orient any endpoints.
\item Does not orient any endpoints.
\item Rule 1 orients $O_2 \circline \!\!* O_4$ as $O_2 - \!\! * O_4$ and $O_2 \circline \!\!* O_3$ as $O_2 - \!\! * O_3$. Then Rule 4 orients $O_1 \circline \!\! * O_3$ as $O_1 \leftarrow \!\! * O_3$ and $O_4 \circline \!\! * O_5$ as $O_4 \leftarrow \!\! * O_5$. Next, Rule 1 again fires twice to orient $O_3 \circlinecirc O_4$ as $O_3 - O_4$. Finally Rule 3 also fires twice to orient $O_2 * \!\! \linecirc O_4$ and $O_2 * \!\! \linecirc O_3$ as $O_2 * \!\! - O_4$ and $O_2 * \!\! - O_3$, respectively. The orientation rules in Step \ref{alg_OR} therefore yield Figure \ref{fig_trace_f}.
\end{enumerate}

Now we would expect CCD to output a pretty good partially oriented MAAG, given that the directed graph contains only one latent variable and no selection variables. However, CCD outputs the graph in Figure \ref{fig_trace_f}. The output contains one error ($O_1$ is not an ancestor of $O_2$) and eight un-oriented endpoints which were oriented by CCI.

\section{Algorithm Details} \label{sec_details}

We present the details of CCI. We claim that statements made herein hold for both cyclic and acyclic directed graphs, unless indicated otherwise. Most of the proofs are located in Section \ref{sec_proofs}.

\subsection{Step 1: Skeleton Discovery}

We first discover the skeleton of an MAAG by consulting a CI oracle. The following result demonstrates that we can discover the skeleton of an MAAG, if we can search over all possible separating sets:
\begin{lemma} \label{lem_IP_dsep}
There exists an inducing path between $O_i$ and $O_j$ if and only if $O_i$ and $O_j$ are d-connected given $\bm{W} \cup \bm{S}$ for all possible subsets $\bm{W} \subseteq \bm{O} \setminus \{O_i, O_j \}$.
\end{lemma}
\noindent Searching over all separating sets is however inefficient. We therefore consider the following sets instead:
\begin{definition} (D-SEP Set)
We say that $O_k \in \textnormal{D-SEP}(O_i,O_j)$ in a directed graph $\mathbb{G}$ if and only if there exists a sequence of observables $\Pi=\langle O_i, \dots , O_k\rangle$ in $\textnormal{Anc}(\{O_i,O_j\}\cup \bm{S})$ such that, for any subpath $\langle O_{h-1}, O_h, O_{h+1} \rangle$ on $\Pi$, we have an inducing path between $O_{h-1}$ and $O_h$ that is into $O_h$ as well as an inducing path between $O_{h+1}$ and $O_h$ that is into $O_h$.
\end{definition}
\noindent Notice that $\textnormal{D-SEP}(O_i,O_j)$ and $\textnormal{D-SEP}(O_j,O_i)$ may not be equivalent. The D-SEP set is important, because we can use it to discover inducing paths without searching over all possible separating sets:
\begin{lemma} \label{lem_dsep}
If there does not exist an inducing path between $O_i$ and $O_j$, then $O_i$ and $O_j$ are d-separated given $\textnormal{D-SEP}(O_i,O_j) \cup \bm{S}$. Likewise, $O_i$ and $O_j$ are d-separated given $\textnormal{D-SEP}(O_j,O_i) \cup \bm{S}$.
\end{lemma}

The D-SEP sets are however not computable. We therefore consider computable possible d-separating sets, which are supersets of the D-SEP sets:
\begin{definition} (Possible D-Separating Set)
We say that $O_k \in \textnormal{PD-SEP}(O_i)$ in any partial oriented mixed graph $\widetilde{\mathbb{G}}$ if and only if there exists a path $\Pi$ between $O_i$ and $O_k$ in $\widetilde{\mathbb{G}}$ such that, for every subpath $\langle O_{h-1}, O_h, O_{h+1} \rangle$ on $\Pi$, either $O_h$ is a v-structure or $\langle O_{h-1}, O_h, O_{h+1} \rangle$ forms a triangle. \label{def_pdsep}
\end{definition}
\noindent The following lemma shows that we can utilize PD-SEP sets in replace of D-SEP sets:
\begin{lemma} \label{lem_pdsep}
If there does not exist an inducing path between $O_i$ and $O_j$, then $O_i$ and $O_j$ are d-separated given $\bm{W} \cup \bm{S}$ with $\bm{W} \subseteq \textnormal{PD-SEP}(O_i)$ in the MAAG $\mathbb{G}^{\prime}$. Likewise, $O_i$ and $O_j$ are d-separated given some $\bm{W} \cup \bm{S}$ with $\bm{W} \subseteq \textnormal{PD-SEP}(O_j)$ in $\mathbb{G}^{\prime}$.
\end{lemma}
\noindent Recall that a similar lemma was also proven in the acyclic case \citep{Spirtes00}. We conclude that the procedure for discovering the skeleton of an MAAG is equivalent to that of an MAG in the acyclic case.

The justification of Step \ref{alg_skeleton} in Algorithm \ref{alg_CCI} then follows by generalizing the above lemma to $\mathbb{G}^{\prime\prime}$, the partially oriented mixed graph discovered by PC's skeleton and FCI's v-structure discovery procedures utilized in Step \ref{alg_skeleton}:
\begin{lemma}\label{lem_pdsep2}
If an inducing path does not exist between $O_i$ and $O_j$ in $\mathbb{G}$, then $O_i$ and $O_j$ are d-separated given $\bm{W} \cup \bm{S}$ with $\bm{W} \subseteq \textnormal{PD-SEP}(O_i)$ in $\mathbb{G}^{\prime\prime}$. Likewise, $O_i$ and $O_j$ are d-separated given some $\bm{W} \cup \bm{S}$ with $\bm{W} \subseteq \textnormal{PD-SEP}(O_j)$ in $\mathbb{G}^{\prime\prime}$.
\end{lemma}
\noindent The above lemma holds because $\textnormal{PD-SEP}(O_i)$ formed using $\mathbb{G}^\prime$ is a subset of $\textnormal{PD-SEP}(O_i)$ formed using $\mathbb{G}^{\prime\prime}$; likewise for $\textnormal{PD-SEP}(O_j)$.

\subsection{Steps 2 \& 3: Short and Long Range Non-Ancestral Relations}

We orient endpoints during v-structure discovery with the following lemma:
\begin{lemma} \label{lem_VS1}
Consider a set $\bm{W} \subseteq \bm{O} \setminus \{O_i, O_j\}$. Now suppose that $O_i$ and $O_k$ are d-connected given $\bm{W} \cup \bm{S}$, and that $O_j$ and $O_k$ are d-connected given $\bm{W} \cup \bm{S}$. If $O_i$ and $O_j$ are d-separated given $\bm{W} \cup \bm{S}$ such that $O_k \not \in \bm{W}$, then $O_k$ is not an ancestor of $\{O_i, O_j\} \cup \bm{S}$.
\end{lemma}

Recall that, if there exists an inducing path between $O_i$ and $O_k$ as well as an inducing path between $O_j$ and $O_k$, then $O_i$ and $O_k$ are d-connected given $\bm{W} \cup \bm{S}$ and $O_j$ and $O_k$ are d-connected given $\bm{W} \cup \bm{S}$ by Lemma \ref{lem_IP_dsep}. Moreover, if $O_i$ and $O_j$ are d-separated given $\bm{W} \cup \bm{S}$, then an inducing path does not exist between $O_i$ and $O_j$. This means that, if we are dealing with the partially oriented MAAG in Figure \ref{fig_VS_a}, and $O_i$ and $O_j$ are d-separated given $\bm{W} \cup \bm{S}$ such that $O_k \not \in \bm{W}$, then we orient the endpoints in Figure \ref{fig_VS_a} as the arrowheads in Figure \ref{fig_VS_b}. Lemma \ref{lem_VS1} therefore justifies the v-structure discovery procedure in Step \ref{alg_vstruc} of CCI for short range non-ancestral relations.

\begin{figure}
\centering
\begin{subfigure}{.25\linewidth}
\centering
\resizebox{\linewidth}{!}{
\begin{tikzpicture}[scale=1.0, shorten >=1pt,auto,node distance=2.8cm, semithick]
                    
\tikzset{vertex/.style = {shape=circle,draw,inner sep=0.4pt}}
 
\node[vertex] (1) at  (0,1) {$O_i$};
\node[vertex] (2) at  (1.5,1) {$O_j$};
\node[vertex] (3) at  (0.75,-0.3) {$O_k$};

\draw[o-o] (1) to (3);
\draw[o-o] (2) to (3);
\end{tikzpicture}
}
\caption{}  \label{fig_VS_a}
\end{subfigure}
\begin{subfigure}{.25\linewidth}
\centering
\resizebox{\linewidth}{!}{
\begin{tikzpicture}[scale=1.0, shorten >=1pt,auto,node distance=2.8cm, semithick]
                    
\tikzset{vertex/.style = {shape=circle,draw,inner sep=0.4pt}}
 
\node[vertex] (1) at  (0,1) {$O_i$};
\node[vertex] (2) at  (1.5,1) {$O_j$};
\node[vertex] (3) at  (0.75,-0.3) {$O_k$};

\draw[o->,> = latex'] (1) to (3);
\draw[o->,> = latex'] (2) to (3);
\end{tikzpicture}
}
\caption{}  \label{fig_VS_b}
\end{subfigure}
\caption{An unshielded triple in (a) oriented to a v-structure in (b) after Step \ref{alg_vstruc} of CCI.}
\end{figure}
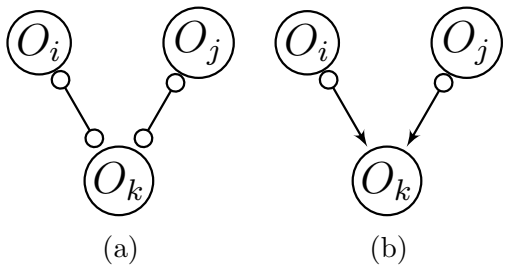

Now Lemma \ref{lem_VS1} also justifies Step \ref{alg_addV} of CCI, because $O_i$ and $O_k$ as well as $O_j$ and $O_k$ need not be adjacent in the underlying MAAG. In fact, $O_k$ may be located far from $O_i$ and $O_j$. CCI utilizes such long range relations because two cyclic directed graphs may agree ``locally'' on d-separation relations, but disagree on some d-separation relations between distant variables \citep{Richardson94}.

\subsection{Step 4: Discovering Non-Minimal D-Separating Sets}

The algorithm now utilizes the graph from Step \ref{alg_addV} in order to find additional d-separating sets. The skeleton discovery phase of CCI finds minimal d-separating sets, but non-minimal d-separating sets can also inform the algorithm about the underlying cyclic causal graph. Recall that, if we have $O_i * \!\! \rightarrow O_j \leftarrow \!\! * O_k$ in the acyclic case, then $O_i$ and $O_j$ are d-connected given $O_j \cup \bm{W} \cup \bm{S}$ for any $\bm{W} \subseteq \bm{O} \setminus \{O_i, O_j\}$ \citep{Spirtes96}. The same fact is not true however in the cyclic case. Consider for example the graph in Figure \ref{fig_SEMIE}. Here, we know that $X_1$ and $X_4$ are d-separated given $\emptyset$, but they are also d-separated given $\{X_2, X_3\}$. Moreover, we have $O_1 \rightarrow O_3 \leftarrow O_4$ in the corresponding MAAG in Figure \ref{fig_SEMIE_MAAG}.

The above example motivates us to search for additional d-separating sets, which will prove to be important for orientating additional circle endpoints as evidenced in the next section. From Lemma \ref{lem_pdsep}, we already know that some subset of $\textnormal{PD-SEP}(O_i)$ and some subset of $\textnormal{PD-SEP}(O_k)$ d-separate $O_i$ and $O_k$ when we additionally condition on $\bm{S}$. We therefore search for the additional d-separating sets by testing all subsets of $\textnormal{PD-SEP}(O_i)$ as well as those of $\textnormal{PD-SEP}(O_k)$.

We summarize the details of Step \ref{alg_addD} in Algorithm \ref{alg_CCI_E}. The sub-procedure specifically works as follows. For each v-structure $O_i * \!\! \rightarrow O_j \leftarrow \!\!* O_k$, the algorithm determines whether $O_i$ and $O_k$ are d-separated given specific supersets of the minimal separating set $\textnormal{Sep}(O_i,O_k)$. In particular, Algorithm \ref{alg_CCI_E} forms the sets $\bm{T}=\textnormal{Sep}(O_i,O_k)\cup O_j \cup \bm{W}$ where $\bm{W} \subseteq \textnormal{PD-SEP}(O_i)\setminus \{\textnormal{Sep}(O_i,O_k) \cup \{O_j, O_k\}\}$ in line \ref{step4_T}. The algorithm then consults the CI oracle with $\bm{T} \cup \bm{S}$ in line \ref{step4_CI}. Finally, like the skeleton discovery phase, the algorithm first consults the CI oracle with the smallest subsets and then progresses to larger subsets until such a separating set $\bm{T} \cup \bm{S}$ is found or all subsets of $\textnormal{PD-SEP}(O_i)\setminus \{\textnormal{Sep}(O_i,O_k) \cup \{O_j, O_k\}\}$ have been exhausted.

\begin{algorithm}[]
 \KwData{$\widehat{\mathcal{G}}$, CI oracle}
 \KwResult{$\widehat{\mathcal{G}}$, \textnormal{SupSep}}
 \BlankLine
 
$m = 0$\\
\Repeat{all ordered triples $\langle O_i, O_j, O_k \rangle$ with the v-structure $O_i * \!\! \rightarrow O_j \leftarrow \!\!* O_k$ have $|\textnormal{PD-SEP}(O_i)| < m$}{
	\Repeat{ all triples $\langle O_i, O_j, O_k \rangle$ with the v-structure $O_i * \!\! \rightarrow O_j \leftarrow \!\!* O_k$ and $|\textnormal{PD-SEP}(O_i)| \geq m$ have been selected}{
    	select the ordered triple $\langle O_i, O_j, O_k \rangle$ with the v-structure $O_i * \!\! \rightarrow O_j \leftarrow \!\!* O_k$ such that $|\textnormal{PD-SEP}(O_i)| \geq m$\\
        \Repeat{all subsets $\bm{W} \subseteq \textnormal{PD-SEP}(O_i) \setminus \{\textnormal{Sep}(O_i,O_k) \cup \{O_j, O_k\}\}$ have been considered or a d-separating set of $O_i$ and $O_k$ has been recorded in $\textnormal{SupSep}(O_i,O_j,O_k)$}{
        	select a subset $\bm{W} \subseteq \textnormal{PD-SEP}(O_i) \setminus \{\textnormal{Sep}(O_i,O_k) \cup \{O_j, O_k\}\}$ with $m$ vertices\\
            $\bm{T} = \bm{W} \cup  \textnormal{Sep}(O_i,O_k) \cup O_j$ \\ \label{step4_T}
        	if $O_i$ and $O_k$ are d-separated given $\bm{T} \cup \bm{S}$, then record the set $\bm{T}$ in $\textnormal{SupSep}(O_i,O_j,O_k)$ \label{step4_CI}
        }
    }
}

 \BlankLine

 \caption{Step \ref{alg_addD} of CCI} \label{alg_CCI_E}
\end{algorithm}

\subsection{Step 5: Orienting with Non-Minimal D-Separating Sets}

The following lemma justifies Step \ref{alg_step5} which utilizes the non-minimal d-separating sets discovered in the previous step:
\begin{lemma} \label{lem_sup_anc2}
Consider a quadruple of vertices $\langle O_i, O_j, O_k, O_l \rangle$. Suppose that we have:
\begin{enumerate}
\item $O_i$ and $O_k$ non-adjacent;
\item $O_i * \!\! \rightarrow O_l \leftarrow \!\! * O_k$;
\item $O_i$ and $O_k$ are d-separated given some $\bm{W} \cup \bm{S}$ with $O_j \in \bm{W}$ and $\bm{W} \subseteq \bm{O} \setminus \{O_i, O_k\}$;
\item $O_j * \!\! \linecirc O_l$
\end{enumerate}
If $O_l \not \in \bm{W}=\textnormal{Sep}(O_i,O_k)$, then we have $O_j * \!\! \rightarrow O_l$. If $O_l \in \bm{W}=\textnormal{SupSep}(O_i,O_j,$ $O_k)$ and in addition we have $O_i * \!\! \rightarrow O_j \leftarrow \!\! * O_k$, then we have $O_j * \!\! - O_l$. 
\end{lemma}
\noindent Notice that the above lemma utilizes $\textnormal{SupSep}(O_i,O_j,O_k)$ as discovered in Step \ref{alg_addD}.

\subsection{Step 6: Long Range Ancestral Relations}

We can justify Step \ref{alg_step6} with the following result:
\begin{lemma} \label{lem_sup_anc1}
If $O_i$ and $O_k$ are d-separated given $\bm{W} \cup \bm{S}$, where $\bm{W} \subseteq \bm{O} \setminus \{ O_i, O_k\}$, and $\bm{Q} \subseteq \textnormal{Anc}(\{O_i, O_k\} \cup \bm{W} \cup \bm{S}) \setminus \{O_i, O_k \}$, then $O_i$ and $O_k$ are also d-separated given $\bm{Q} \cup \bm{W} \cup \bm{S}$.
\end{lemma}
\noindent Notice that the above lemma allows us to infer long range ancestral relations because all variables in $\bm{Q}$ are ancestors of $\{O_i, O_k\} \cup \bm{W} \cup \bm{S}$. Step \ref{alg_step6} of Algorithm \ref{alg_CCI} then follows by the contrapositive of Lemma \ref{lem_sup_anc1}: 
\begin{corollary}
Assume that $O_i$ and $O_k$ are d-separated by $\bm{W} \cup \bm{S}$ with $O_j \in \bm{W}$ and $\bm{W} \subseteq \bm{O} \setminus \{ O_i, O_k\}$, but $O_i$ and $O_k$ are d-connected by $O_l \cup \bm{W} \cup \bm{S}$. Then, $O_l$ is not an ancestor of $O_j \cup \bm{S}$.
\end{corollary}

\subsection{Step 7: Orientation Rules} \label{sec_OR}
We will now describe the orientation rules in Step \ref{alg_OR}. Notice that the orientation rules are always applied after Step \ref{alg_step6} and therefore also after v-structure discovery. This ordering implies that, if $O_i$ and $O_j$ are non-adjacent and we have $O_i * \!\! - \!\! * O_k * \!\! - \!\! * O_j$, but we do not have $O_i * \!\! \rightarrow O_k \leftarrow \!\! * O_j$, then $O_k \in \textnormal{Sep}(O_i, O_j)$; this follows because, if $O_k \not \in \textnormal{Sep}(O_i, O_j)$, then we would have $O_i * \!\! \rightarrow O_k \leftarrow \!\! * O_j$ by v-structure discovery.

\subsubsection{First to Third Orientation Rules}
Lemma \ref{lem_VS1} allows us to infer non-ancestral relations. The following lemma allows us to infer ancestral relations:
\begin{lemma} \label{lem_or_1}
Suppose that there is a set $\bm{W} \setminus \{O_i, O_k\}$ and every proper subset $\bm{V} \subset \bm{W}$ d-connects $O_i$ and $O_k$ given $\bm{V} \cup \bm{S}$. If $O_i$ and $O_k$ are d-separated given $\bm{W} \cup \bm{S}$ where $O_j \in \bm{W}$, then $O_j$ is an ancestor of $\{O_i, O_k\} \cup \bm{S}$.
\end{lemma}
\noindent The above lemma justifies the following orientation rule:
\begin{lemma}
If we have $O_i * \!\! \rightarrow O_j \circline \!\! * O_k$ with $O_i$ and $O_k$ non-adjacent, then orient $O_j \circline \!\! * O_k$ as $O_j - \!\! * O_k$.
\end{lemma}
\begin{proof}
If we have $O_i * \!\! \rightarrow O_j \circline \!\!* O_k$ with $O_i$ and $O_k$ non-adjacent, then $O_j \in \textnormal{Sep}(O_i, O_k)$ because we have already performed v-structure discovery. By Lemma \ref{lem_or_1}, we know that $O_j \in \textnormal{Anc}(\{O_i, O_k\} \cup \bm{S})$. We more specifically know that $O_j \in \textnormal{Anc}(O_k)$ because the arrowhead $O_i * \!\! \rightarrow O_j$ implies that $O_j \not \in \textnormal{Anc}(O_i \cup \bm{S})$.
\end{proof}
\noindent For example, if we have the structure in Figure \ref{fig_OR1_a}, then we can add a undirected edge as in Figure \ref{fig_OR1_b}. 

\begin{figure}
\centering
\begin{subfigure}{.25\linewidth}
\centering
\resizebox{\linewidth}{!}{
\begin{tikzpicture}[scale=1.0, shorten >=1pt,auto,node distance=2.8cm, semithick]
                    
\tikzset{vertex/.style = {shape=circle,draw,inner sep=0.4pt}}
 
\node[vertex] (1) at  (0,1) {$O_i$};
\node[vertex] (2) at  (1.5,1) {$O_k$};
\node[vertex] (3) at  (0.75,-0.3) {$O_j$};

\draw[o->,> = latex'] (1) to (3);
\draw[o-o] (2) to (3);
\end{tikzpicture}
}
\caption{}  \label{fig_OR1_a}
\end{subfigure}
\begin{subfigure}{.25\linewidth}
\centering
\resizebox{\linewidth}{!}{
\begin{tikzpicture}[scale=1.0, shorten >=1pt,auto,node distance=2.8cm, semithick]
                    
\tikzset{vertex/.style = {shape=circle,draw,inner sep=0.4pt}}
 
\node[vertex] (1) at  (0,1) {$O_i$};
\node[vertex] (2) at  (1.5,1) {$O_k$};
\node[vertex] (3) at  (0.75,-0.3) {$O_j$};

\draw[o->,> = latex'] (1) to (3);
\draw[o-] (2) to (3);
\end{tikzpicture}
}
\caption{}  \label{fig_OR1_b}
\end{subfigure}

\begin{subfigure}{.25\linewidth}
\centering
\resizebox{\linewidth}{!}{
\begin{tikzpicture}[scale=1.0, shorten >=1pt,auto,node distance=2.8cm, semithick]
                    
\tikzset{vertex/.style = {shape=circle,draw,inner sep=0.4pt}}
 
\node[vertex] (1) at  (0,1) {$O_i$};
\node[vertex] (2) at  (1.5,1) {$O_k$};
\node[vertex] (3) at  (0.75,-0.3) {$O_j$};
\node[vertex] (4) at  (0.75,-1.6) {$O_l$};

\draw[o->,> = latex'] (1) to (3);
\draw[o-] (2) to (3);
\draw[o->, > = latex', bend right] (1) to (4);
\draw[o-, bend right] (4) to (2);
\draw[o-o] (3) to (4);
\end{tikzpicture}
}
\caption{}  \label{fig_OR1_c}
\end{subfigure}
\begin{subfigure}{.25\linewidth}
\centering
\resizebox{\linewidth}{!}{
\begin{tikzpicture}[scale=1.0, shorten >=1pt,auto,node distance=2.8cm, semithick]
                    
\tikzset{vertex/.style = {shape=circle,draw,inner sep=0.4pt}}
 
\node[vertex] (1) at  (0,1) {$O_i$};
\node[vertex] (2) at  (1.5,1) {$O_k$};
\node[vertex] (3) at  (0.75,-0.3) {$O_j$};
\node[vertex] (4) at  (0.75,-1.6) {$O_l$};

\draw[o->,> = latex'] (1) to (3);
\draw[-] (2) to (3);
\draw[o->, > = latex', bend right] (1) to (4);
\draw[-, bend right] (4) to (2);
\draw[-] (3) to (4);
\end{tikzpicture}
}
\caption{}  \label{fig_OR1_d}
\end{subfigure}
\caption{The first part of Rule 1 orients the graph (a) to (b). Note that the edge $O_i \circarrow O_j$ is potentially 2-triangulated w.r.t. $O_k$ in (c), so the second part of Rule 1 cannot fire. However, if we additionally have $O_j - O_k$, then we can orient 3 endpoints with Rule 3 and obtain (d).}
\end{figure}
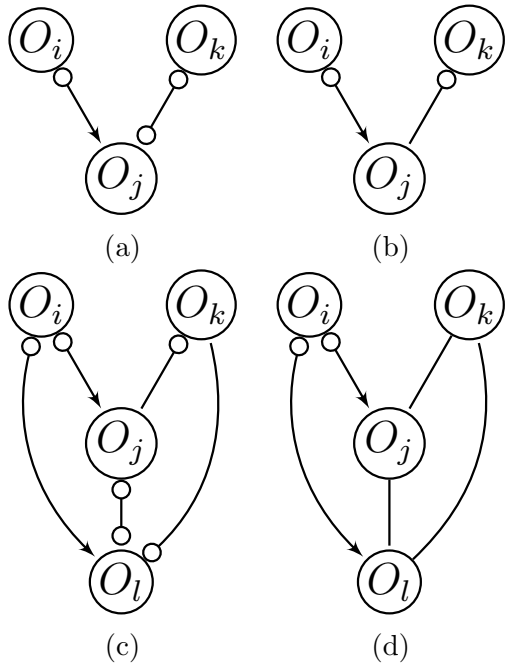

We may also add an arrowhead at $O_k$ in Figure \ref{fig_OR1_b} provided that some additional conditions are met. The following lemma is central to causal discovery with cycles:
\begin{lemma} \label{lem_triang_main}
If we have $O_i * \!\! \rightarrow O_j \text{---} O_k$ with $O_i$ and $O_k$ non-adjacent, then $O_i * \!\! \rightarrow O_j$ is in a triangle involving $O_i,O_j$ and $O_l$ ($l \not = k$) with $O_j \text{---} O_l$ and $O_i * \!\! \rightarrow O_l$. Moreover, there exists a sequence of undirected edges between $O_l$ and $O_k$ that does not include $O_j$.
\end{lemma}
\noindent The above statement may appear arcane at first glance, but it justifies multiple orientation rules.

The following definitions are useful towards applying Lemma \ref{lem_triang_main}:

\begin{definition}
(Potentially Undirected Path) A potentially undirected path $\Pi$ exists between $O_i$ and $O_j$ if and only if all endpoints on $\Pi$ are tails or circles.
\end{definition}

\begin{definition}
(Potential 2-Triangulation) The edge $O_i * \!\! - \!\! * O_j$ is said to be potentially 2-triangulated w.r.t. $O_k$ if and only if (1) $O_i, O_j$ and another vertex $O_l$ is in a triangle, (2) we have $O_j \textit{---} O_l$, $O_j \circline O_l$, $O_j \linecirc O_k$ or $O_j \circlinecirc O_l$, (3) we have $O_i * \!\! \rightarrow O_l$ or $O_i * \!\! \linecirc O_l$, and (4) there exists a potentially undirected path between $O_l$ and $O_k$ that does not include $O_j$.
\end{definition}

We now have three orientation rules that utilize the concept of potential 2-triangulation:
\begin{lemma} \label{lem_OR123}
The following orientation rules are sound:
\begin{enumerate} [label=\#\arabic*.]
\setcounter{enumi}{0}
\item If we have $O_i * \!\! \rightarrow O_j \circline \!\! * O_k$ with $O_i$ and $O_k$ non-adjacent, then orient $O_j \circline \!\! * O_k$ as $O_j - \!\! * O_k$. Furthermore, if $O_i * \!\! \rightarrow O_j$ is not potentially 2-triangulated w.r.t. $O_k$, then orient $O_j \linecirc O_k$ as $O_j \rightarrow O_k$.
\item If we have $O_i - \!\! * O_j \circline \!\!* O_k$ with $O_i$ and $O_k$ nonadjacent, and $O_j \circline \!\!* O_k$ is not potentially 2-triangulated w.r.t. $O_i$, then orient $O_j \circline \!\!* O_k$ as $O_j - \!\!* O_k$.
\item Suppose that we have $O_i *\!\! \rightarrow O_j - O_k$ with $O_i$ and $O_k$ nonadjacent, and $O_i *\!\! \rightarrow O_j$ is potentially 2-triangulated w.r.t. $O_k$. If $O_i *\!\!\rightarrow O_j$ can be potentially 2-triangulated w.r.t. $O_k$ using only one vertex $O_l$ in the triangle involve $\{O_i,O_j,O_l\}$, then orient $O_i * \!\! \linecirc O_l$ as $O_i * \!\! \rightarrow O_l$, $O_j \circline \!\! * O_l$ as $O_j - \!\! * O_l$ and/or $O_j * \!\! \linecirc O_l$ as $O_j * \!\! - O_l$. Next, if there exists only one potentially undirected path $\Pi_{O_lO_k}$ between $O_l$ and $O_k$, then substitute all circle endpoints on $\Pi_{O_lO_k}$ with tail endpoints.
\end{enumerate}
\end{lemma}
\begin{proof}
The following arguments correspond to their associated orientation rule:
\begin{enumerate} [label=\#\arabic*.]
\setcounter{enumi}{0}
\item The first part follows from Lemma \ref{lem_or_1}. The second part follows by the contrapositive of Lemma \ref{lem_triang_main}.
\item Suppose that we have $O_i - O_j$ and $O_j \leftarrow \!\!* O_k$. But this would contradict Lemma \ref{lem_triang_main}. Suppose instead that we had $O_i \rightarrow O_j$ and $O_j \leftarrow \!\!* O_k$. But the arrowheads give rise to another contradiction because we know that $O_j \in \textnormal{Sep}(O_i,O_k)$, so $O_j \in \textnormal{Anc}(\{O_i,O_k\} \cup \bm{S})$ by Lemma \ref{lem_or_1}.
\item Follows directly from Lemma \ref{lem_triang_main}.
\end{enumerate}
\end{proof}
\noindent For example, $O_i \circarrow O_j$ is not potentially triangulated in Figure \ref{fig_OR1_a} (there are only three variables), so we may orient the endpoint $O_j \linecirc O_k$ as $O_j \rightarrow O_k$ according to the first orientation rule. On the other hand, $O_i \circarrow O_j$ is potentially triangulated w.r.t $O_k$ in Figure \ref{fig_OR1_c}, so we cannot orient the circle endpoint at $O_k$ as an arrowhead. However, if we additionally have $O_j - O_k$, then we can apply Rule 3 to orient $O_j \circlinecirc O_l$ as $O_j - O_l$ and $O_l \circline O_k$ as $O_l - O_k$ to ultimately obtain Figure \ref{fig_OR1_d}.

\subsubsection{Fourth \& Fifth Orientation Rules}

\begin{figure}
\centering
\begin{subfigure}{.7\linewidth}
\centering
\resizebox{\linewidth}{!}{
\begin{tikzpicture}[scale=1.0, shorten >=1pt,auto,node distance=2.8cm, semithick]
                    
\tikzset{vertex/.style = {shape=circle,draw,inner sep=0.4pt}}
 
\node[vertex] (0) at  (-1.6,1) {$O_i$};
\node[vertex] (1) at  (0,1) {$O_j$};
\node[vertex] (2) at  (1.6,1) {$O_k$};
\node[vertex] (5) at  (0,2) {$O_l$};

\tikzset{vertex/.style = {shape=circle,inner sep=0.4pt}}

\node[vertex] (3) at  (1.2,1) {$*$};
\node[vertex] (6) at  (-1.2,1) {$*$};
\node[vertex] (7) at  (-1.26,1.2) {$*$};
\node[vertex] (8) at  (0.31,1.8) {$*$};

\node[vertex] (4) at  (2.5,1) {$\implies$};

\draw[->,> = latex'] (-1.1,1) to (1);
\draw[-] (1) to (1.15,1);
\draw[-] (0.40,1.78) to (2);
\draw[o-] (5) to (-1.18,1.23);

\tikzset{vertex/.style = {shape=circle,draw,inner sep=0.4pt}}
 
\node[vertex] (9) at  (3.4,1) {$O_i$};
\node[vertex] (10) at  (5,1) {$O_j$};
\node[vertex] (11) at  (6.6,1) {$O_k$};
\node[vertex] (14) at  (5,2) {$O_l$};

\tikzset{vertex/.style = {shape=circle,inner sep=0.4pt}}

\node[vertex] (12) at  (6.2,1) {$*$};
\node[vertex] (15) at  (3.8,1) {$*$};
\node[vertex] (16) at  (3.74,1.2) {$*$};
\node[vertex] (17) at  (5.31,1.8) {$*$};

\draw[->,> = latex'] (3.9,1) to (10);
\draw[-] (10) to (6.15,1);
\draw[-] (5.40,1.78) to (11);
\draw[<-,> = latex'] (14) to (3.82,1.23);
\end{tikzpicture}
}
\caption{}  \label{fig_OR4}
\end{subfigure}

\begin{subfigure}{.7\linewidth}
\centering
\resizebox{\linewidth}{!}{
\begin{tikzpicture}[scale=1.0, shorten >=1pt,auto,node distance=2.8cm, semithick]
                    
\tikzset{vertex/.style = {shape=circle,draw,inner sep=0.4pt}}
 
\node[vertex] (0) at  (-1.6,1) {$O_i$};
\node[vertex] (1) at  (0,1) {$O_j$};
\node[vertex] (2) at  (1.6,1) {$O_k$};

\tikzset{vertex/.style = {shape=circle,inner sep=0.4pt}}

\node[vertex] (3) at  (1.2,1) {$*$};
\node[vertex] (6) at  (-0.43,1) {$*$};
\node[vertex] (8) at  (1.37,1.35) {$*$};

\node[vertex] (4) at  (2.5,1) {$\implies$};

\draw[-,> = latex'] (0) to (-0.48,1);
\draw[-] (1) to (1.15,1);
\draw[o-, bend left=60] (0) to (1.35,1.4);

\tikzset{vertex/.style = {shape=circle,draw,inner sep=0.4pt}}
 
\node[vertex] (9) at  (3.4,1) {$O_i$};
\node[vertex] (10) at  (5,1) {$O_j$};
\node[vertex] (11) at  (6.6,1) {$O_k$};

\tikzset{vertex/.style = {shape=circle,inner sep=0.4pt}}

\node[vertex] (12) at  (6.2,1) {$*$};
\node[vertex] (15) at  (4.57,1) {$*$};
\node[vertex] (17) at  (6.37,1.35) {$*$};

\draw[-,> = latex'] (9) to (4.52,1);
\draw[-] (10) to (6.15,1);
\draw[-, bend left=60] (9) to (6.35,1.4);
\end{tikzpicture}
}
\caption{}  \label{fig_OR5}
\end{subfigure}
\caption{Examples of Rules 4 and 5 in (a) and (b), respectively.}
\end{figure}
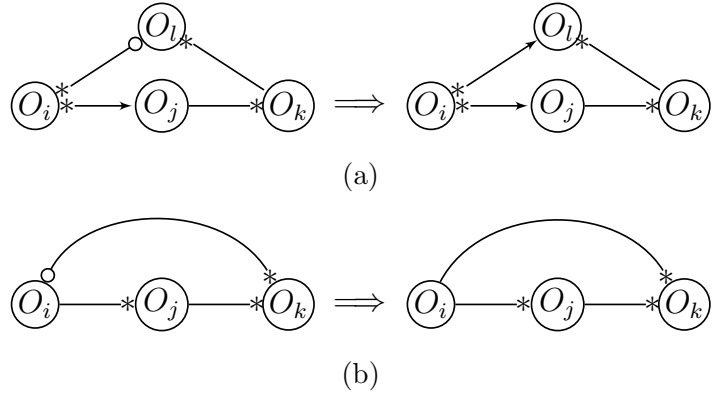

\begin{lemma} \label{lem_OR45}
The following orientation rules are sound:
\begin{enumerate} [label=\#\arabic*.]
\setcounter{enumi}{3}
\item If $O_i *\!\! \rightarrow O_j - \!\! * O_k$, there exists a path $\Pi = \langle O_k, \cdots, O_i \rangle$ with at least $n \geq 3$ vertices such that we have $O_h - \!\! * O_{h+1}$ for all $1 \leq h \leq n-1$ except for only one index $l$ where we have $O_l \circline \!\! * O_{l+1}$, then orient $O_l \circline \!\! * O_{l+1}$ as $O_l \leftarrow \!\! * O_{l+1}$.
\item If we have the sequence of vertices $\langle O_1, \dots, O_n \rangle$ such that $O_i - \!\! * O_{i+1}$ with $ 1 \leq i \leq n-1$, and we have $O_1 \circline \!\! * O_n$, then orient  $O_1 \circline \!\! * O_n$ as $O_1 - \!\! * O_n$.
\end{enumerate}
\end{lemma}
\begin{proof}
The following arguments correspond to their associated orientation rule:
\begin{enumerate} [label=\#\arabic*.]
\setcounter{enumi}{3}
\item Suppose for a contradiction that we had $O_l - \!\! * O_{l+1}$. But then $O_j$ is an ancestor of $O_i \cup \bm{S}$ by transitivity of the tails. 

\item Follows by transitivity of the tail.
\end{enumerate}
\end{proof}
\noindent We provide examples of Rules 4 and 5 in Figures \ref{fig_OR4} and \ref{fig_OR5}, respectively.

\subsubsection{Sixth \& Seventh Orientation Rules}

The CCI algorithm has 2 more orientation rules which require successive applications of the first orientation rule. We first require the following definition:
\begin{definition}
(Non-Potentially 2-Triangulated Path) A path $\Pi = \langle O_1, \dots, O_n \rangle$ is said to be non-potentially 2-triangulated if the following conditions hold:
\begin{enumerate}
\item If $n \geq 3$, then the vertices $O_{i-1}$ and $O_{i+1}$ are non-adjacent for every $2 \leq i \leq n-1$ (i.e., every consecutive triple is non-adjacent), and $O_{i-1} * \!\! - \!\! * O_i * \!\! - \!\! * O_{i+1}$ is a non-v-structure for every $2 \leq i \leq n-1$.
\item If $n \geq 4$, then the vertices $O_{i} *\!\! - \!\! * O_{i+1}$ are not potentially 2-triangulated w.r.t. $O_{i+2}$ for every $1 \leq i \leq n-3$.
\end{enumerate}
\end{definition}

\begin{figure}
\centering
\begin{subfigure}{.5\linewidth}
\centering
\resizebox{\linewidth}{!}{
\begin{tikzpicture}[scale=1.0, shorten >=1pt,auto,node distance=2.8cm, semithick]
                    
\tikzset{vertex/.style = {shape=circle,draw,inner sep=0.4pt}}
 
\node[vertex] (1) at  (0,1) {$O_i$};
\node[vertex] (2) at  (1.6,1) {$O_k$};

\tikzset{vertex/.style = {shape=circle,inner sep=0.4pt}}

\node[vertex] (3) at  (1.2,1) {$*$};
\node[vertex] (4) at  (2.5,1) {$\implies$};

\draw[o-] (1) to (1.15,1);
\draw[->,> = latex', bend left=80,dashed] (1) to (2);

\tikzset{vertex/.style = {shape=circle,draw,inner sep=0.4pt}}
 
\node[vertex] (5) at  (3.4,1) {$O_i$};
\node[vertex] (6) at  (5,1) {$O_k$};

\tikzset{vertex/.style = {shape=circle,inner sep=0.4pt}}

\node[vertex] (7) at  (4.6,1) {$*$};

\draw[-] (5) to (4.55,1);
\draw[->,> = latex', bend left=80,dashed] (5) to (6);
\end{tikzpicture}
}
\caption{}  \label{fig_OR6}
\end{subfigure}

\begin{subfigure}{.54\linewidth}
\centering
\resizebox{\linewidth}{!}{
\begin{tikzpicture}[scale=1.0, shorten >=1pt,auto,node distance=2.8cm, semithick]
                    
\tikzset{vertex/.style = {shape=circle,draw,inner sep=0.4pt}}
 
\node[vertex] (1) at  (0.75,1) {$O_k$};
\node[vertex] (2) at  (-0.2,0) {$O_j$};
\node[vertex] (3) at  (1.7,0) {$O_l$};
\node[vertex] (4) at  (0.75,-1.5) {$O_i$};

\tikzset{vertex/.style = {shape=circle,inner sep=0.4pt}}
\node[vertex] (5) at  (0.75,0.59) {$*$};
\node[vertex] (6) at  (1.05,0.67) {$*$};
\node[vertex] (7) at  (0.45,0.67) {$*$};

\node[vertex] (8) at  (2.6,0) {$\implies$};

\draw[o-] (4) to (0.75,0.5);
\draw[-,> = latex'] (3) to (1.1,0.63);
\draw[-,> = latex'] (2) to (0.4,0.63);
\draw[<-, dashed,> = latex', bend right] (2) to (4);
\draw[<-, dashed,> = latex', bend left] (3) to (4);

\tikzset{vertex/.style = {shape=circle,draw,inner sep=0.4pt}}
 
\node[vertex] (9) at  (4.45,1) {$O_k$};
\node[vertex] (10) at  (3.5,0) {$O_j$};
\node[vertex] (11) at  (5.4,0) {$O_l$};
\node[vertex] (12) at  (4.45,-1.5) {$O_i$};

\tikzset{vertex/.style = {shape=circle,inner sep=0.4pt}}
\node[vertex] (13) at  (4.45,0.59) {$*$};
\node[vertex] (14) at  (4.75,0.67) {$*$};
\node[vertex] (15) at  (4.15,0.67) {$*$};

\draw[-] (12) to (4.45,0.5);
\draw[-,> = latex'] (11) to (4.8,0.63);
\draw[-,> = latex'] (10) to (4.1,0.63);
\draw[<-, dashed,> = latex', bend right] (10) to (12);
\draw[<-, dashed,> = latex', bend left] (11) to (12);

\end{tikzpicture}
}
\caption{}  \label{fig_OR7}
\end{subfigure}
\caption{Rules 6 and 7 in (a) and (b), respectively. Dotted lines indicate non-potentially 2-triangulated paths.}
\end{figure}
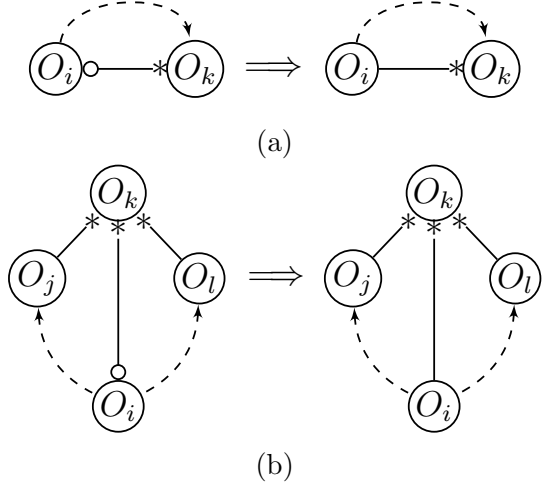

The following orientation rules utilize the above definition:
\begin{lemma} \label{lem_OR67}
The following orientation rules are sound:
\begin{enumerate} [label=\#\arabic*.]
\setcounter{enumi}{5}
\item If we have $O_k * \!\! \linecirc O_i$, there exists a non-potentially 2-triangulated path $\Pi = \langle O_i, O_j, O_l, \dots O_k \rangle$ such that $O_k * \!\! \linecirc O_i$ is not potentially 2-triangulated w.r.t. $O_j$, and $O_j * \!\! - \!\! * O_i * \!\! - \!\! * O_k$ is a non-v-structure, then orient $O_k * \!\! \linecirc O_i$ as $O_k * \!\! - O_i$ (Figure \ref{fig_OR6}).
\item Suppose we have $O_i \circline \!\! * O_k$, $O_j - \!\! * O_k * \!\! - O_l$, a non-potentially 2-triangulated path $\Pi_1$ from $O_i$ to $O_j$, and a non-potentially 2-triangulated path $\Pi_2$ from $O_i$ to $O_l$. Let $O_m$ be a vertex adjacent to $O_i$ on $\Pi_1$ ($O_m$ could be $O_j$), and let $O_n$ be the vertex adjacent to $O_i$ on $\Pi_2$ ($O_n$ could be $O_l$). If further $O_m *\!\! - \!\! * O_i * \!\! - \!\! * O_n$ is a non-v-structure and $O_i \circline \!\! * O_k$ is not potentially 2-triangulated w.r.t. both $O_n$ and $O_m$, then orient $O_i \circline \!\! * O_k$ as $O_i - \!\! * O_k$ (Figure \ref{fig_OR7}).
\end{enumerate}
\end{lemma}
\begin{proof}
The following arguments apply to their corresponding orientation rules:
\begin{enumerate} [label=\#\arabic*.]
\setcounter{enumi}{5}
\item Suppose for a contradiction that we have $O_i \leftarrow \!\! * O_k$. Then we can iteratively apply the first orientation rule on $\Pi$ until the transitivity of the added tails contradicts the arrowhead at $O_i$.

\item Suppose for a contradiction that we have $O_i \leftarrow \!\! * O_k$. Then we iteratively apply the first orientation rule along $\Pi_1$ or $\Pi_2$ (or both). In any case, $O_i \in \textnormal{Anc}(O_k \cup \bm{S})$ by transitivity of the added tails which contradicts the arrowhead at $O_i$.
\end{enumerate}
\end{proof}

\section{Experiments} \label{sec_exp}

We now report the empirical results.

\subsection{Synthetic Data}

We generated 1000 random Gaussian directed cyclic graphs (directed graphs with at least one cycle) with an expected neighborhood size $\mathbb{E}(N)=2$ and $p=20$ vertices using the following procedure. First, we generated a random adjacency matrix $B$ with independent realizations of Bernoulli$(\mathbb{E}(N)/(2p$ $-2))$ random variables in the off-diagonal entries. We then replaced the non-zero entries in $B$ with independent realizations of Uniform($[-1,-0.1]\cup[0.1,1]$) random variables. 
We can interpret a nonzero entry $B_{ij}$ as an edge from $X_i$ to $X_j$ with
coefficient $B_{ij}$ in the following linear model:
\begin{equation}
\begin{aligned}
&X_i = \sum_{r=1}^{p} B_{ir}X_r + \varepsilon_i,
\end{aligned}
\end{equation}
for $i = 1, \dots , p$ where $\varepsilon_1, . . ., \varepsilon_p$ are mutually independent $\mathcal{N}(0, 1)$ random variables. The variables $X_1,\dots,$ $X_p$ then have a multivariate Gaussian distribution with mean vector $0$ and covariance matrix $\Sigma = (\mathbb{I} - B)^{-1}(\mathbb{I} - B)^{-T}$, where $\mathbb{I}$ is the $p \times p$ identity matrix.

We similarly generated 1000 random Gaussian DAGs with the same parameters but created each random adjacency matrix $B$ with independent realizations of Bernoulli$(\mathbb{E}(N)/(p-1))$ random variables in the lower triangular and off-diagonal entries \citep{Colombo12}. 

We introduced latent and selection variables into each DCG and DAG as follows. We first randomly selected a set of 0-3 latent common causes $\bm{L}$ without replacement. We then selected a set of 0-3 selection variables $\bm{S}$ without replacement from the set of vertices $\bm{X} \setminus \bm{L}$ with at least two parents. 

We ultimately created datasets with sample sizes of 500, 1000, 5000, 10000, 50000 and 100000 for each of the 1000 DCGs and each of the 1000 DAGs. We therefore generated a total of $1000 \times 6 \times 2 = 12000$ datasets.

\subsection{Algorithms}

We compared the following four CB algorithms. We also list each algorithm's assumptions:
\begin{enumerate}
\item CCI: acyclic or cyclic with linear SEM-IE
\item FCI: acyclic
\item RFCI: acyclic
\item CCD: acyclic or cyclic with linear SEM-IE, no latent variables\footnote{CCD cannot handle selection bias as proposed in \citep{Richardson99}, but the algorithm may be able to if we modify the proofs.}
\end{enumerate}
All algorithms additionally assume d-separation faithfulness. Only CCI remains sound under CLS. We ran all algorithms using Fisher's z-test with $\alpha$ set to 1E-2 for sample sizes 500 and 1000, 1E-3 for 5000 and 10000, and 1E-4 otherwise. Recall that we require decreasing p-values with increasing sample sizes in order to ensure consistency \citep{Kalisch07,Colombo12}.

\subsection{Metrics}
We assessed the algorithms using the structural Hamming distance (SHD) \citep{Tsamardinos06} to the corrected oracle graphs. We construct the corrected oracle graph as follows. First, we run an algorithm with a CI oracle to obtain the oracle graph. Then, we replace any incorrect arrowhead with a tail and vice versa. For example, if we have $O_i * \!\! \rightarrow O_j$ in the oracle graph, but $O_j \in \textnormal{Anc}(O_i \cup \bm{S})$, then we replace $O_i * \!\! \rightarrow O_j$ with $O_i * \!\! - O_j$. An algorithm which is sound will always output an oracle graph which does not require correction, if the algorithm's assumptions are satisfied.

\begin{figure*}
\centering
\begin{subfigure}{0.49\textwidth}
  \centering
  \includegraphics[width=1\linewidth]{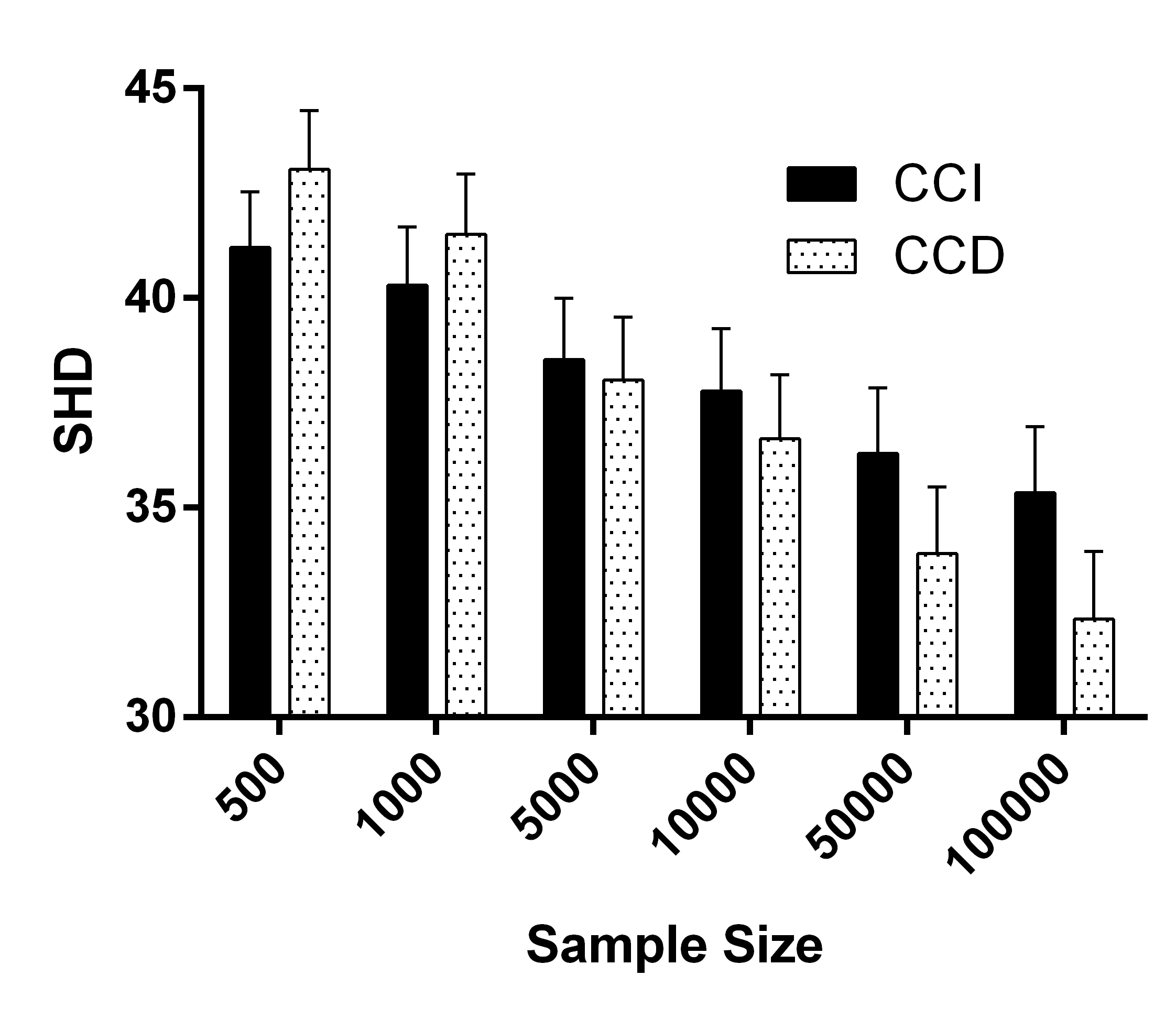}
  \caption{}
  \label{fig_cyc:SHD_cyc}
\end{subfigure}
\begin{subfigure}{0.49\textwidth}
  \centering
  \includegraphics[width=1\linewidth]{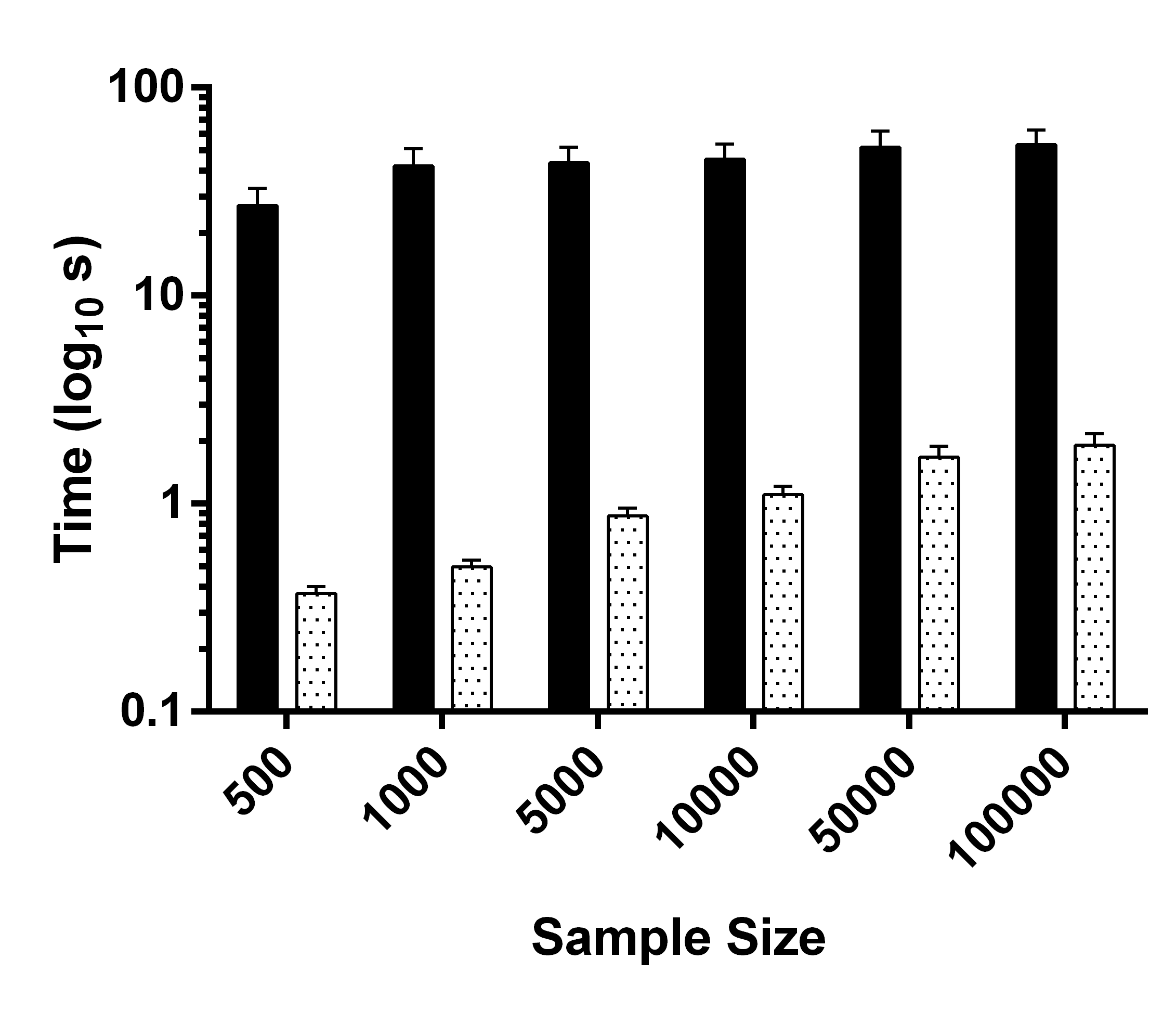}
  \caption{}
  \label{fig_cyc:time}
\end{subfigure}

\begin{subfigure}{0.49\textwidth}
  \centering
  \includegraphics[width=1\linewidth]{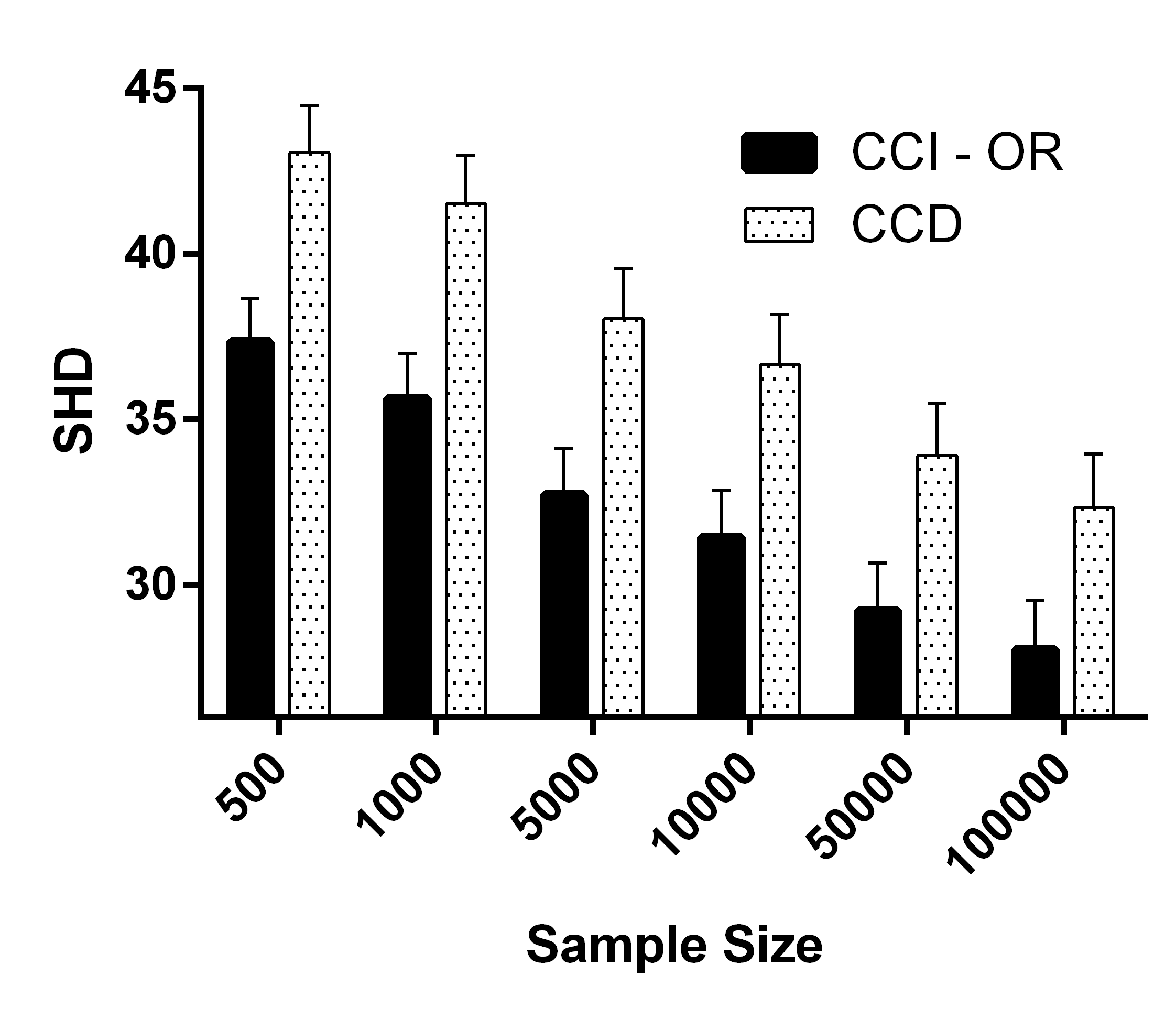}
  \caption{}
  \label{fig_cyc:SHD_no_or}
\end{subfigure}
\begin{subfigure}{0.49\textwidth}
  \centering
  \includegraphics[width=1\linewidth]{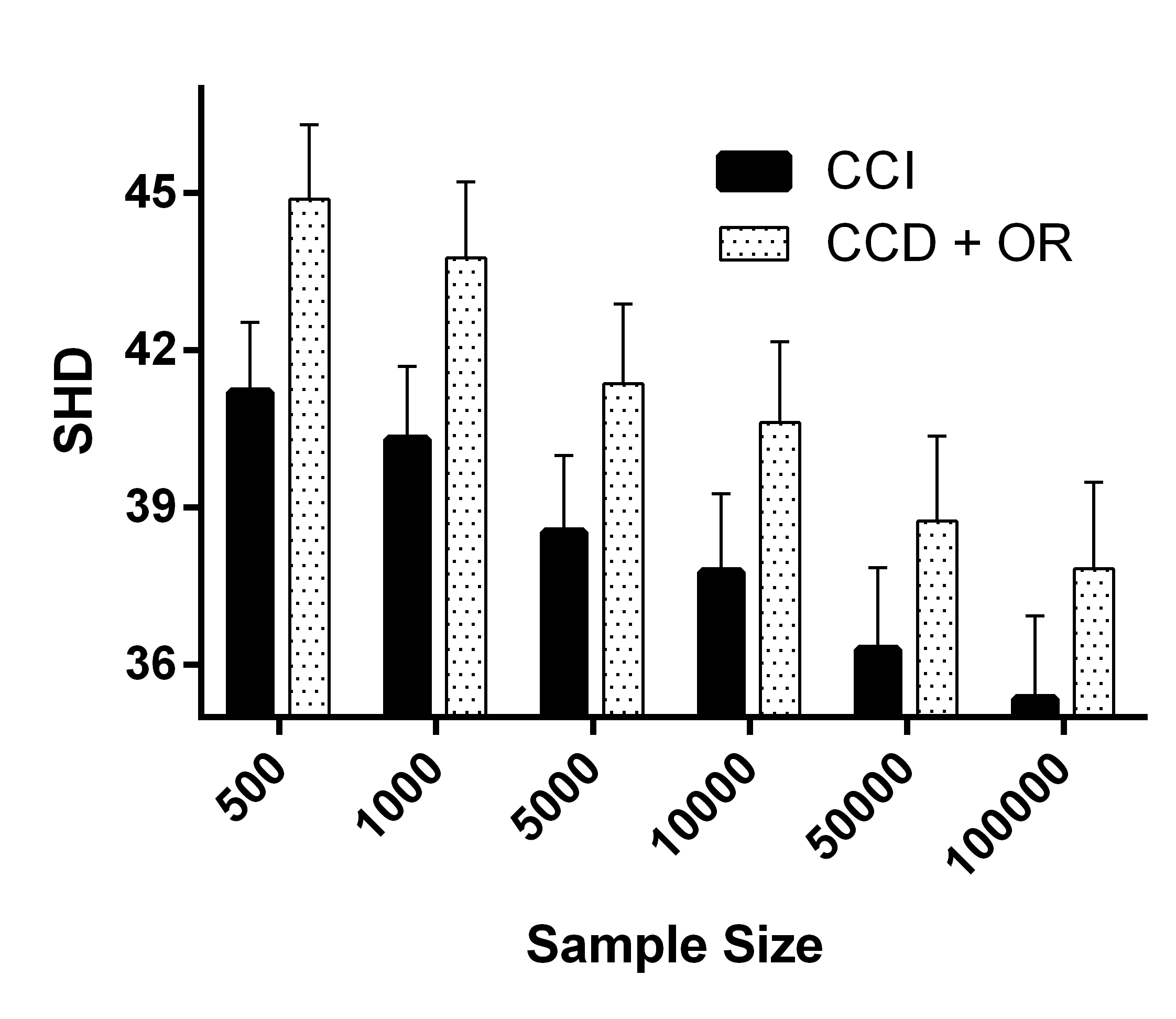}
  \caption{}
  \label{fig_cyc:SHD_or}
\end{subfigure}
\caption{CCI versus CCD on recovering cyclic graphs with latent variables and selection bias. Smaller SHD is better; error bars always denote 95\% confidence intervals of the mean. (a) CCD outperforms CCI on sample sizes $>$1000. (b) CCD also has shorter running times than CCI. However, CCI--OR outperforms CCD in (c). CCI similarly outperforms CCD+OR in (d).} \label{fig_cyc}
\end{figure*}

\subsection{Cyclic Case}

We first compared CCI to CCD in the cyclic case. Here, we hope that CCI will outperform CCD, since CCD cannot handle latent common causes. We have summarized the results in Figure \ref{fig_cyc}. We unexpectedly found that CCD outperformed CCI by a significant margin across the largest four of the 6 sample sizes (Figure \ref{fig_cyc:SHD_cyc}; min t = -2.94, p = 3.34E-3). CCD also completed in a much shorter time frame than CCI (Figure \ref{fig_cyc:time}). Recall however that CCI makes more long range inferences than CCD by applying multiple orientation rules. We therefore also analyzed the performance of CCI with the orientation rules removed, denoted as CCI minus OR (CCI--OR); this comparison pits CCI against CCD on more fair grounds, because CCD does not have orientation rules. Here, we found that CCI--OR outperformed CCD across all sample sizes (max t = -26.64, p$<$2.2E-16; Figure \ref{fig_cyc:SHD_no_or}). We also added the orientation rules of CCI to CCD, which we call CCD plus OR (CCD+OR). CCI again outperformed CCD+OR across all sample sizes (max t = -13.13, p$<$2.2E-16; Figure \ref{fig_cyc:SHD_or}). We conclude that CCI outperforms CCD once we account for the orientation rules.

\subsection{Acyclic Case}

\begin{figure}
\centering
\begin{subfigure}{0.49\textwidth}
  \centering
\includegraphics[width=1\linewidth]{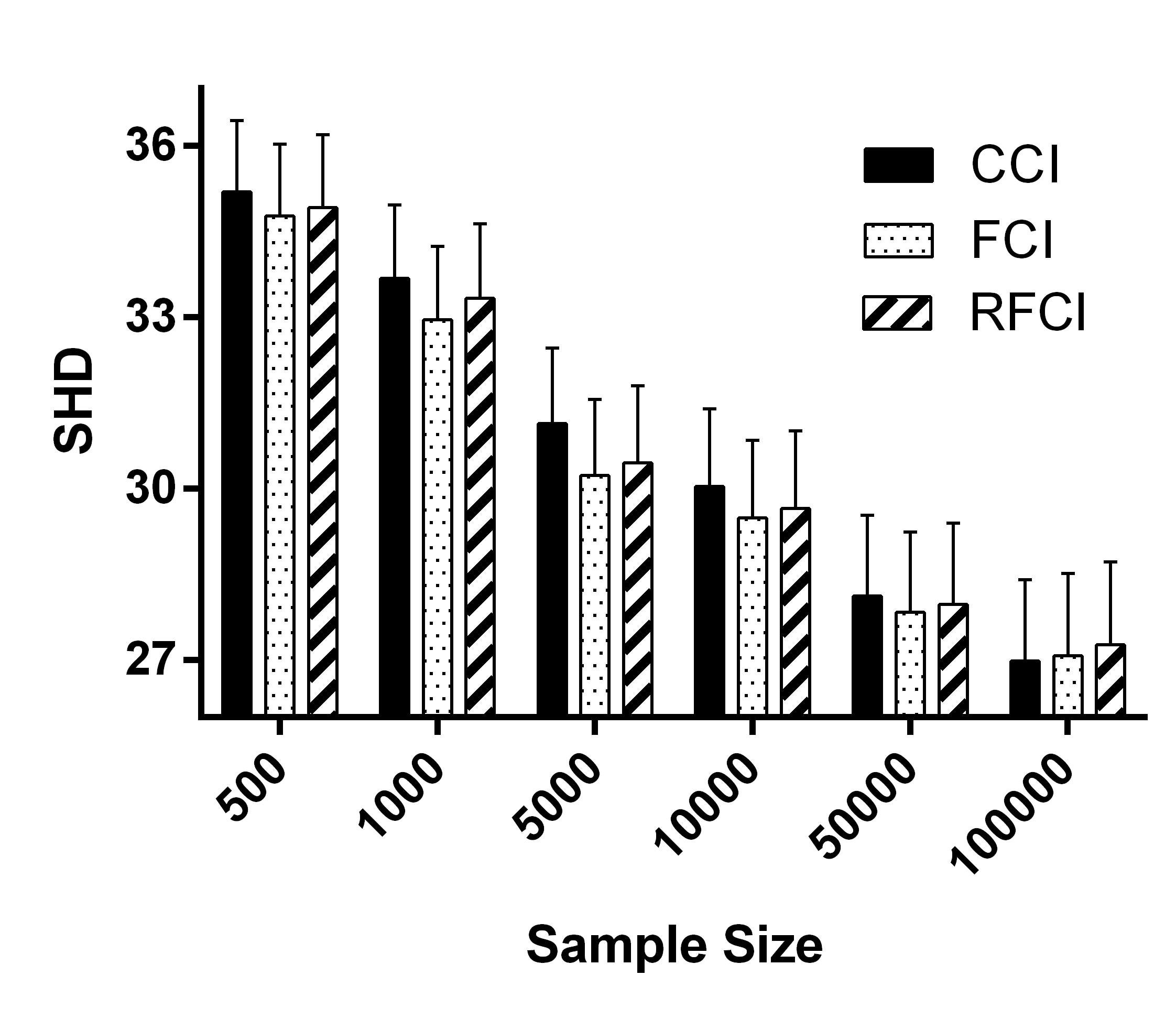}
\caption{} \label{fig_acyc:SHD}
\end{subfigure}
\begin{subfigure}{0.49\textwidth}
  \centering
\includegraphics[width=1\linewidth]{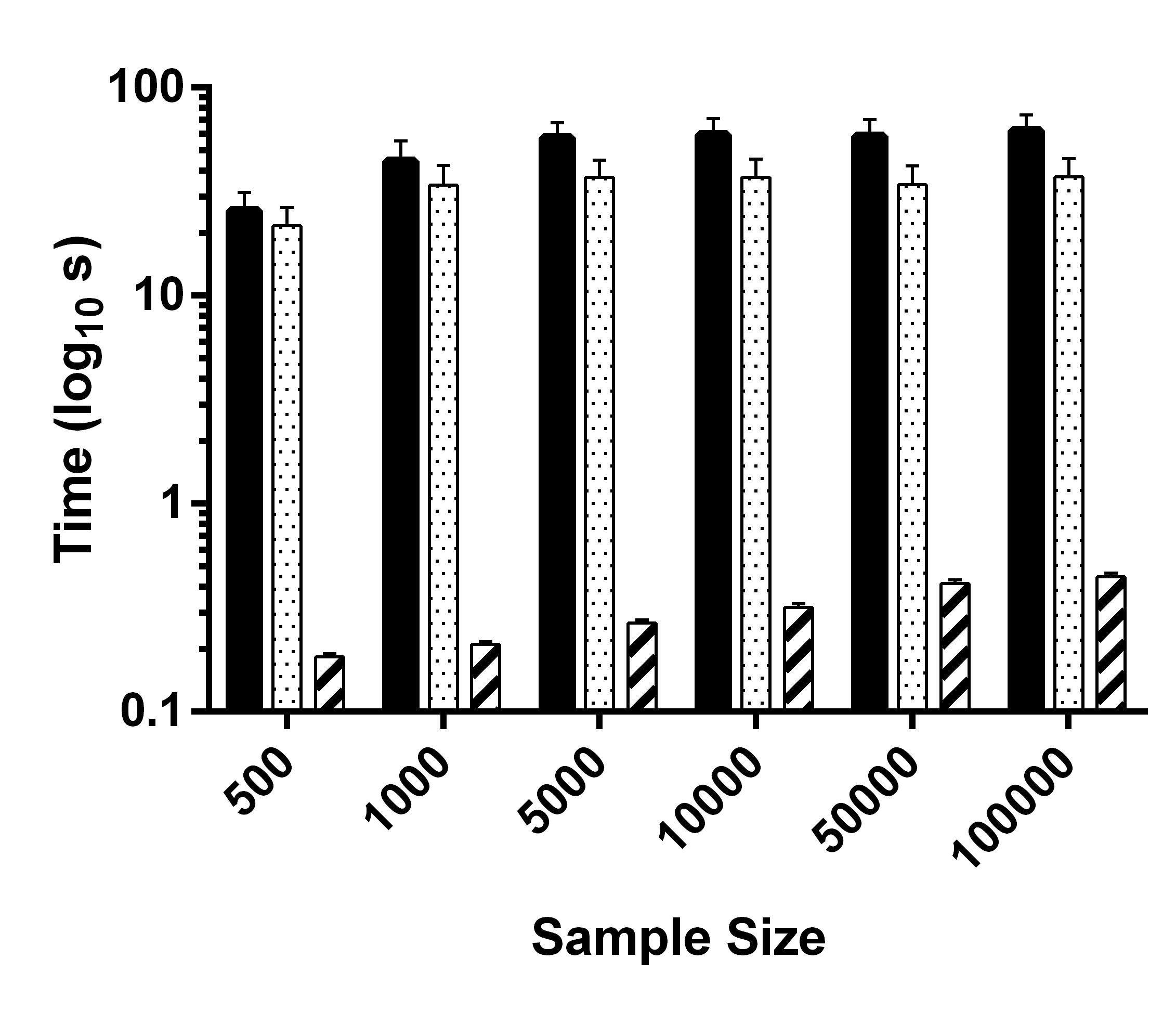}
\caption{} \label{fig_acyc:time}
\end{subfigure}

\begin{subfigure}{0.49\textwidth}
  \centering
\includegraphics[width=1\linewidth]{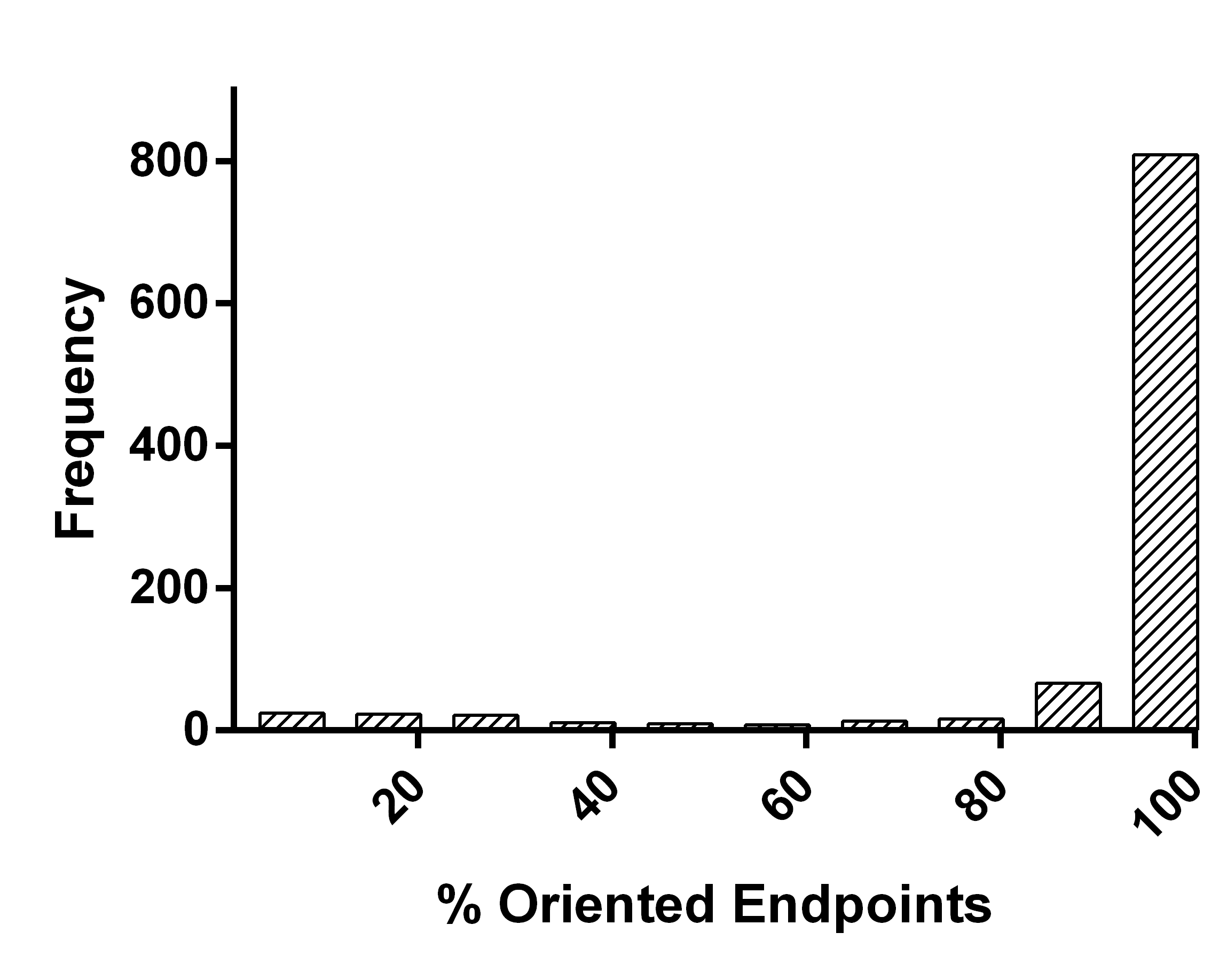}
\caption{}\label{fig_acyc:hist}
\end{subfigure}
\caption{CCI versus FCI and RFCI in recovering acyclic graphs with latent variables and selection bias. (a) FCI and RFCI outperform CCI by a slight margin on sample sizes $\leq 10000$. (b) CCI takes slightly longer to complete than FCI. (c) CCI orients the majority of endpoints oriented by FCI.} 
\label{fig_acyc}
\end{figure}

We next compared CCI to FCI and RFCI in the acyclic case. Here, we expect CCI to perform worse than FCI and RFCI on average, because CCI does not assume acyclicity. However, we hope that CCI will not underperform by a large margin.

We have summarized the SHD results in Figure \ref{fig_acyc:SHD} and the timing results in Figure \ref{fig_acyc:time}. CCI recovered acyclic causal graphs less accurately than FCI by a significant margin with sample sizes $\leq 10000$ (min t = 4.23, p = 2.59E-5). We found no statistically significant difference at larger sample sizes $(p>0.05/6)$. CCI was also outperformed by RFCI with sample sizes between $1000$ to $10000$ (min t = 2.78, p = 5.50E-3). The effect sizes were nonetheless very small; CCI had mean SHDs at most $0.91$ points greater than FCI and RFCI across all sample sizes. We conclude that CCI underperforms FCI and RFCI in the acyclic cause but only by a negligible margin.

We also sought to answer the follow question: how many edges does CCI orient compared to FCI in the acyclic case? It is impossible for CCI to orient 100\% of the endpoints oriented by FCI, because FCI assumes acyclicity whereas CCI does not. We would however ideally like CCI to orient most of the endpoints oriented by FCI in the acyclic case. 

In order to answer the question, we ran the CCI and FCI algorithms on 1000 random DAGs with a CI oracle. CCI oriented 89.98\% (SE: 0.74\%) of the endpoints oriented by FCI on average. Moreover, the histogram of percentages had a heavy left skew (Figure \ref{fig_acyc:hist}), so the median of CCI vs FCI was 100\%. We conclude that CCI orients the majority of endpoints oriented by FCI in the acyclic case.

\subsection{Real Data}

\begin{figure*}
\centering
\begin{subfigure}{0.49\textwidth}
  \centering
  \includegraphics[width=1\linewidth]{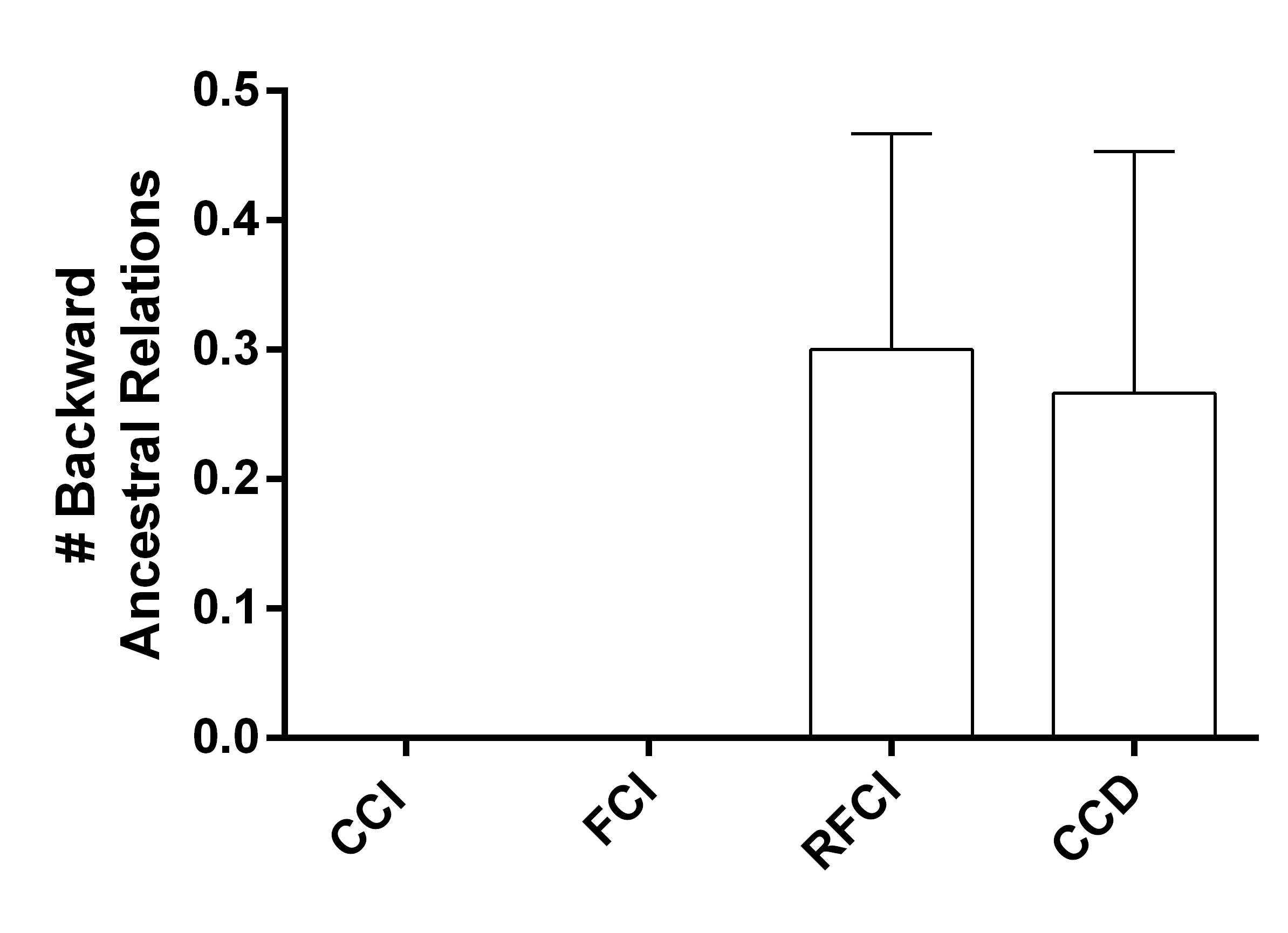}
  \caption{}
  \label{fig_real:anc}
\end{subfigure}
\begin{subfigure}{0.49\textwidth}
  \centering
  \includegraphics[width=1\linewidth]{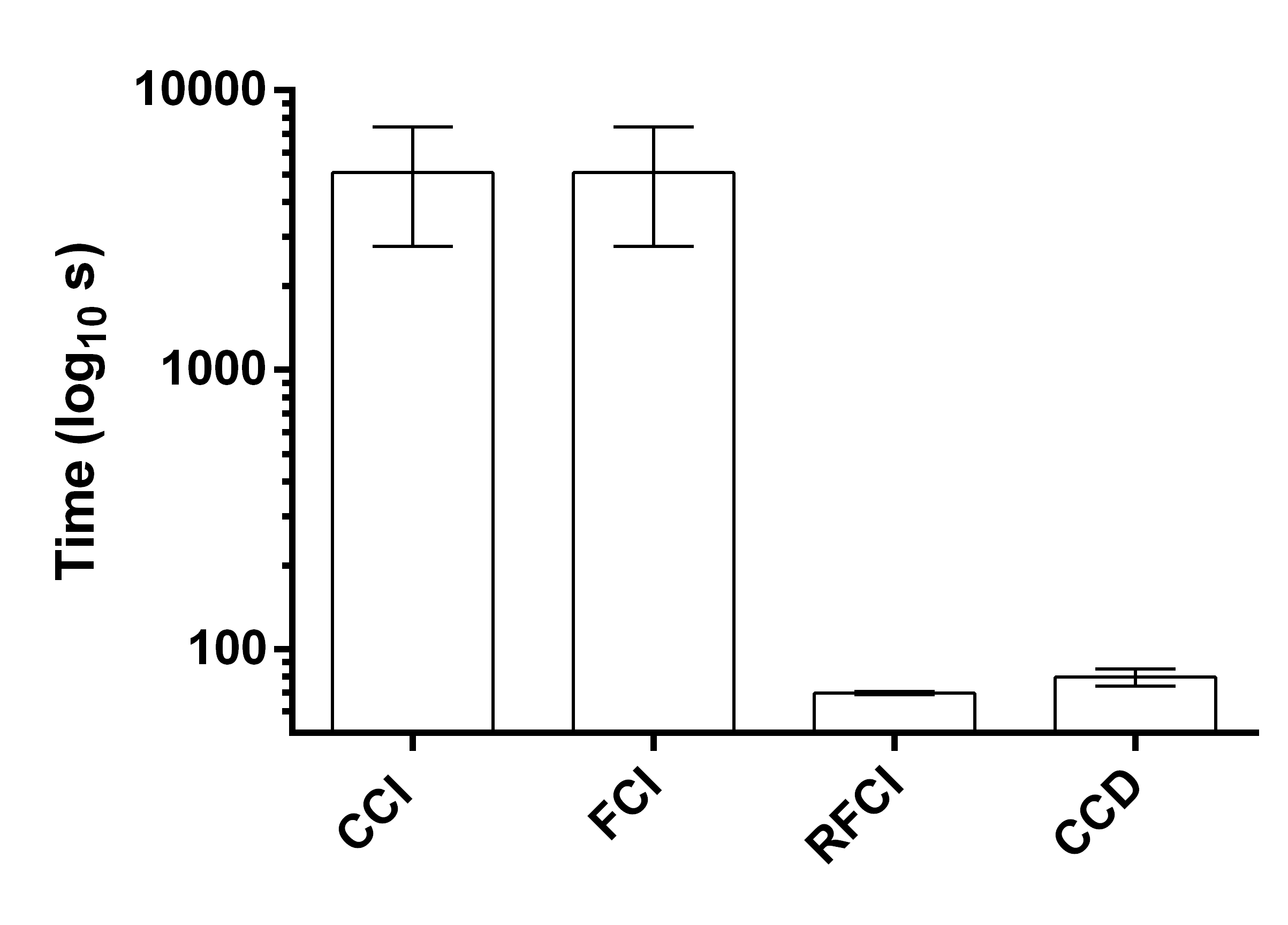}
  \caption{}
  \label{fig_real:time}
\end{subfigure}

\caption{Results of algorithms on real data. (a) FCI and CCI perform the best because they do not discover any backward ancestral relations. (b) However, both of these algorithms take much longer to complete than the others. } \label{fig_real}
\end{figure*}

We finally ran the same algorithms using the nonparametric CI test called RCoT \citep{Strobl17} at $\alpha=0.01$ on a publicly available longitudinal dataset from the Framingham Heart Study \citep{Mahmood14}, where scientists measured a variety of clinical variables related to cardiac health. The dataset contains 28 variables, 3 waves and 2008 samples after performing list-wise deletion. All but 2 variables were measured in all 3 waves.

Note that we do not have access to a gold standard solution set in this case. We can however develop an approximate solution set by utilizing time information because we cannot have ancestral relations directed backwards in time. A variable in wave $b$ thus cannot be an ancestor of a variable in wave $a<b$. In terms of a partially oriented MAAG, this means that any edge between wave $a$ and wave $b$ with both a tail and an arrowhead at a vertex in wave $b$ is incorrect. We thus evaluated the algorithms using the average number of incorrect ancestral relations directed backwards in time.

We have summarized the results in Figure \ref{fig_real:anc} averaged over 30 bootstrapped datasets. Notice that FCI and CCI outperform CCD by a significant margin because FCI and CCI do not make any errors (t=-2.80, p=8.90E-3). Moreover, RFCI performs the worst, highlighting the price one must pay for speed (Figure \ref{fig_real:time}). We conclude that accounting for latent variables allows for more accurate causal discovery on this dataset.

\section{Conclusion} \label{sec_conc}

This report introduced an algorithm called CCI for performing causal discovery with CLS provided that we can represent the cyclic causal process as a linear SEM-IE. As far as I am aware, CCI is the most general CB algorithm proposed to date. The experimental results in the previous section highlight the superior or comparable performance of CCI when compared to previous algorithms that do not allow selection bias, latent variables and/or cycles.

\section*{References}
\bibliography{biblio}

\section{Appendix: Algorithms} \label{sec_algs}

We will utilize ideas developed for the PC, FCI, RFCI and CCD algorithms in order to construct CCI. We therefore briefly review PC, FCI, RFCI and CCD in the next four subsections.

\subsection{The PC Algorithm} \label{sec_PC}

The PC algorithm considers the following problem: assume that $\mathbb{P}$ is d-separation faithful to an unknown DAG $\mathbb{G}$. Then, given oracle information about the conditional independencies between any pair of variables $X_i$ and $X_j$ given any $\bm{W} \subseteq \bm{X}\setminus \{X_i,X_j \}$ in $\mathbb{P}$, reconstruct as much of the underlying DAG as possible. The PC algorithm ultimately accomplishes this goal by reconstructing the DAG up to its \textit{Markov equivalence class}, or the set of DAGs with the same conditional dependence and independence relations between variables in $\bm{X}$ \citep{Spirtes00,Meek95_2}.

The PC algorithm represents the Markov equivalence class of DAGs using a \textit{completed partially directed acyclic graph} (CPDAG). A \textit{partially directed acyclic graph} (PDAG) is a graph with both directed and undirected edges. A PDAG is \textit{completed} when the following conditions hold: (1) every directed edge also exists in every DAG belonging to the Markov equivalence class of the DAG, and (2) there exists a DAG with $X_i \rightarrow X_j$ and a DAG with $X_i \leftarrow X_j$ in the Markov equivalence class for every undirected edge $X_i - X_j$. Each edge in the CPDAG also has the following interpretation:
\begin{enumerate}[label=(\roman*)]
\item An edge (directed or undirected) is absent between two vertices $X_i$ and $X_j$ if and only if there exists some $\bm{W} \subseteq \bm{X}\setminus \{X_i, X_j\}$ such that $X_i \ci X_j | \bm{W}$.
\item If there exists a directed edge from $X_i$ to $X_j$, then $X_i \in \textnormal{Pa}(X_j)$.
\end{enumerate}

\begin{algorithm}[]
 \KwData{CI oracle}
 \KwResult{$\widehat{\mathbb{G}}$, \textnormal{Sep}, $\mathcal{M}$}
 \BlankLine
 
 Form a complete graph $\widehat{\mathbb{G}}$ on $\bm{O}$ with vertices $\circlinecirc$ \\
 $l \leftarrow -1$ \\
 \Repeat{all pairs of adjacent vertices $(O_i, O_j)$ in $\widehat{\mathbb{G}}$ satisfy $|\textnormal{Adj}(O_i)\setminus O_j| \leq l$}{
 Let $l=l+1$ \\
 \Repeat{all ordered pairs of adjacent vertices $(O_i, O_j)$ in $\widehat{\mathbb{G}}$ with $|\textnormal{Adj}(O_i)\setminus O_j| \geq l$ have been considered}{
 \ForAll{vertices in $\widehat{\mathbb{G}}$}{
 	Compute $\textnormal{Adj}(O_i)$
 }
 Select a new ordered pair of vertices $(O_i, O_j)$ that are adjacent in $\widehat{\mathbb{G}}$ and satisfy $|\textnormal{Adj}(O_i)\setminus O_j| \geq l$ \\
 
 \Repeat{$O_i$ and $O_j$ are no longer adjacent in $\widehat{\mathbb{G}}$ or all $\bm{W} \subseteq \textnormal{Adj}(O_i)\setminus O_j$ with $|\bm{W}| = l$ have been considered }{
 
 Choose a new set $\bm{W} \subseteq \textnormal{Adj}(O_i)\setminus O_j$ with $|\bm{W}|=l$ \\
 
 \If{$O_i \ci O_j|\bm{W} \cup \bm{S}$}{
 	Delete the edge $O_i \circlinecirc O_j$ from $\widehat{\mathbb{G}}$ \\
 	Let $\textnormal{Sep}(O_i,O_j) = \textnormal{Sep}(O_j,O_i) = \bm{W}$
 }

 }
 }
 }
 
 Form a list $\mathcal{M}$ of all unshielded triples $\langle O_k,\cdot,O_m \rangle$ (i.e., the middle vertex is left unspecified) in $\widehat{\mathbb{G}}$ with $k < m$

 \BlankLine

 \caption{PC's skeleton discovery procedure} \label{pc_skel}
\end{algorithm}

The PC algorithm learns the CPDAG through a three step procedure. First, the algorithm initializes a fully connected undirected graph and then determines the presence or absence of each undirected edge using the following fact: under d-separation faithfulness, $X_i$ and $X_j$ are non-adjacent if and only if $X_i$ and $X_j$ are conditionally independent given some subset of $\textnormal{Pa}(X_i)\setminus X_j$ or some subset of $\textnormal{Pa}(X_j)\setminus X_i$. Note that PC cannot differentiate between the parents and children of a vertex from its neighbors using an undirected graph. Thus, PC tests whether $X_i$ and $X_j$ are conditionally independent given all subsets of $\textnormal{Adj}(X_i)\setminus X_j$ and all subsets of $\textnormal{Adj}(X_j)\setminus X_i$, where $\textnormal{Adj}(X_i)$ denotes the vertices adjacent to $X_i$ in $\mathbb{G}$ (a superset of $\textnormal{Pa}(X_i)$), in order to determine the final adjacencies; we refer to this sub-procedure of PC as \textit{skeleton discovery} and list the pseudocode in Algorithm \ref{pc_skel}. The PC algorithm therefore removes the edge between $X_i$ and $X_j$ during skeleton discovery if such a conditional independence is found.

Step 2 of the PC algorithm orients unshielded triples to v-structures $X_i \rightarrow X_j \leftarrow X_k$ if $X_j$ is not in the set of variables which rendered $X_i$ and $X_k$ conditionally independent in the skeleton discovery phase of the algorithm. The final step of the PC algorithm involves the repetitive application of three orientation rules to replace as many tails as possible with arrowheads \citep{Meek95_2}.

\subsection{The FCI Algorithm} \label{sec_FCI}

The FCI algorithm considers the following problem: assume that the distribution of $\bm{X} = \bm{O} \cup \bm{L} \cup \bm{S}$ is d-separation faithful to an unknown DAG. Then, given oracle information about the conditional independencies between any pair of variables $O_i$ and $O_j$ given any $\bm{W} \subseteq \bm{O}\setminus \{O_i,O_j \}$ as well as $\bm{S}$, reconstruct as much information about the underlying DAG as possible \citep{Spirtes00}. The FCI algorithm ultimately accomplishes this goal by reconstructing a MAG up to its Markov equivalence class, or the set of MAGs with the same conditional dependence and independence relations between variables in $\bm{O}$ given $\bm{S}$ \citep{Zhang08}.

The FCI algorithm represents the Markov equivalence class of MAGs using a \textit{completed partial maximal ancestral graph} (CPMAG).\footnote{The CPMAG is also known as a partial ancestral graph (PAG). However, we will use the term CPMAG in order to mimic the use of the term CPDAG.} A \textit{partial maximal ancestral graph} (PMAG) is nothing more than a MAG with possibly some circle endpoints. A PMAG is \textit{completed}  (and hence a CPMAG) when the following conditions hold: (1) every tail and arrowhead also exists in every MAG belonging to the Markov equivalence class of the MAG, and (2) there exists a MAG with a tail and a MAG with an arrowhead in the Markov equivalence class for every circle endpoint. Each edge in the CPMAG also has the following interpretations:
\begin{enumerate}[label=(\roman*)]\label{edge_interp_1}
\item An edge is absent between two vertices $O_i$ and $O_j$ if and only if there exists some $\bm{W} \subseteq \bm{O}\setminus \{O_i, O_j\}$ such that $O_i \ci O_j | \bm{W} \cup \bm{S}$. That is, an edge is absent if and only if there does not exist an inducing path between $O_i$ and $O_j$.
\item If an edge between $O_i$ and $O_j$ has an arrowhead at $O_j$, then $O_j \not \in \textnormal{Anc}(O_i \cup \bm{S})$.
\item If an edge between $O_i$ and $O_j$ has a tail at $O_j$, then $O_j \in \textnormal{Anc}(O_i \cup \bm{S})$.
\end{enumerate}

The FCI algorithm learns the CPMAG through a three step procedure involving skeleton discovery, v-structure orientation and orientation rule application. The skeleton discovery procedure involves running PC's skeleton discovery procedure, orienting v-structures using Algorithm \ref{fci_vstruc}, and then re-performing skeleton discovery using possible d-separating sets (see Definition \ref{def_pdsep}) constructed after the v-structure discovery process. FCI then orients v-structures again using Algorithm \ref{fci_vstruc} on the final skeleton. The third step of FCI involves the repetitive application of 10 orientation rules \citep{Zhang08}.

\begin{algorithm}[] \label{fci_vstruc}
 \KwData{$\widehat{\mathbb{G}}$, \textnormal{Sep}, $\mathcal{M}$}
 \KwResult{$\widehat{\mathbb{G}}$}
 \BlankLine
 
 \ForAll{elements $\langle O_i, O_j, O_k \rangle$ in $\mathcal{M}$}{
 	\If{$O_j \not \in \textnormal{Sep}(O_i, O_k)$}{
    	Orient $O_i * \!\! \linecirc O_j \circline \!\! * O_k$ as $O_i* \!\! \rightarrow O_j \leftarrow \!\! * O_k$ in $\widehat{\mathbb{G}}$
    }
 }
 \BlankLine

 \caption{Orienting v-structures} \label{fci_vstruc}
\end{algorithm}

\begin{algorithm}[]
 \KwData{$\widehat{\mathbb{G}}$, Sep}
 \KwResult{$\widehat{\mathbb{G}}$, Sep, $\mathcal{M}$}
 \BlankLine
 
 \ForAll{vertices $O_i$ in $\widehat{\mathbb{G}}$}{
 	Compute $\textnormal{PD-SEP}(O_i)$ \\
    \ForAll{vertices $O_j \in \textnormal{Adj}(O_i)$}{
     Let $l = -1$ \\
 \Repeat{$O_i$ and $O_j$ are no longer adjacent in $\widehat{\mathbb{G}}$ or $|\textnormal{PD-SEP}(O_i)\setminus O_j|<l$ }{
 		Let $l=l+1$ \\
        \Repeat{$O_i$ and $O_j$ are no longer adjacent in $\widehat{\mathbb{G}}$ or all $\bm{W} \subseteq \textnormal{PD-SEP}(O_i)\setminus O_j$ with $|\bm{W}| = l$ have been considered}{
 		Choose a (new) set $\bm{W} \subseteq \textnormal{PD-SEP}(O_i) \setminus {O_j}$ with $|\bm{W}| = l$ \\
         \If{$O_i \ci O_j | \bm{W} \cup \bm{S}$}{
         	Delete edge $O_i * \!\! - \!\! * O_j$ in $\widehat{\mathbb{G}}$\\
            Let \textnormal{Sep}($O_i, O_j$) = \textnormal{Sep}($O_j, O_i$) $= \bm{W}$
         }
 	}
    }
    
    } 
 }
 Reorient all edges in $\widehat{\mathbb{G}}$ as $\circlinecirc$\\
 Form a list $\mathcal{M}$ of all unshielded triples $\langle O_k, \cdot, O_m \rangle$ in $\widehat{\mathbb{G}}$ with $k<m$
 \BlankLine

 \caption{Obtaining the final skeleton in the FCI algorithm}  \label{fci_final_skel}
\end{algorithm} 

\subsection{The RFCI Algorithm}

Discovering inducing paths can require large possible d-separating sets, so the FCI algorithm often takes too long to complete. The RFCI algorithm \citep{Colombo12} resolves this problem by recovering a graph where the presence and absence of an edge have the following modified interpretations:
\begin{enumerate}[label=(\roman*)]\label{edge_interp_2}
\item The absence of an edge between two vertices $O_i$ and $O_j$ implies that there exists some $\bm{W} \subseteq \bm{O}\setminus \{O_i, O_j\}$ such that $O_i \ci O_j | \bm{W} \cup \bm{S}$.
\item The presence of an edge between two vertices $O_i$ and $O_j$ implies that $O_i \not \ci O_j | \bm{W} \cup \bm{S}$ for all $\bm{W} \subseteq \textnormal{Adj}(O_i) \setminus O_j$ and for all $\bm{W} \subseteq \textnormal{Adj}(O_j) \setminus O_i$. Here $\textnormal{Adj}(O_i)$ denotes the set of vertices adjacent to $O_i$ in RFCI's graph.
\end{enumerate}
\noindent We encourage the reader to compare these edge interpretations to the edge interpretations of FCI's CPMAG. 

The RFCI algorithm learns its graph (not necessarily a CPMAG) also through a three step process. The algorithm performs skeleton discovery using PC skeleton discovery procedure (Algorithm \ref{pc_skel}). RFCI then orients v-structures using Algorithm \ref{rfci_vstruc}. Notice that Algorithm \ref{rfci_vstruc} requires more steps than Algorithm \ref{fci_vstruc} used in FCI because an inducing path may not exist between any two adjacent vertices after only running PC's skeleton discovery procedure. RFCI must therefore check for additional conditional dependence relations in order to infer the non-ancestral relations. RFCI finally repetitively applies the 10 orientation rules of FCI in the last step with some modifications to the fourth orientation rule (see \citep{Colombo12} for further details).

\begin{algorithm}[]
 \KwData{Initial skeleton $\widehat{\mathbb{G}}$, Sep, $\mathcal{M}$}
 \KwResult{$\widehat{\mathbb{G}}$, Sep}
 \BlankLine
 
 Let $\mathcal{L}$ denote an empty list \\ \label{rfci_vs_start}
 
 \While{$\mathcal{M}$ is non-empty}{
	 Choose an unshielded triple $ \langle O_i, O_j, O_k \rangle$ from $\mathcal{M}$ \\
     \If{ $O_i \ci O_j | \textnormal{Sep}(O_i,O_k) \cup \bm{S}$ and $O_j \ci O_k | \textnormal{Sep}(O_i,O_k) \cup \bm{S}$}{
     	Add $ \langle O_i, O_j, O_k \rangle$ to $\mathcal{L}$ 
	} \Else{
    	\For{ $r \in \{i,k\}$ }{
        	\If{ $O_r \ci O_j| ( \textnormal{Sep}(O_i,O_k) \setminus O_j ) \cup \bm{S}$}{
            	Find a minimal separating set $\bm{W} \subseteq \textnormal{Sep}(O_i, O_k)$ for $O_r$ and $O_j$ \\
                Let $\textnormal{Sep}(O_r, O_j) = \textnormal{Sep}(O_j, O_r) = \bm{W}$ \\
                Add all triples $\langle O_{\text{min}(r,j)}, \cdot, O_{\text{max}(r,j)} \rangle$ that form a triangle in $\widehat{\mathbb{G}}$ into $\mathcal{M}$ \\
                Delete from $\mathcal{M}$ and $\mathcal{L}$ all triples containing $(O_r,O_j):$ $\langle O_r, O_j,\cdot \rangle$, $\langle O_j, O_r,\cdot \rangle$, $\langle \cdot,O_j, O_r \rangle$ and $\langle \cdot,O_r, O_j \rangle$ \\
                Delete edge $O_r * \!\! - \!\! * O_j$ in $\widehat{\mathbb{G}}$
            }
        
        }
    }
     Remove $\langle O_i, O_j, O_k \rangle$ from $\mathcal{M}$
     }  \label{rfci_vs_end}
    \ForAll{elements $\langle O_i, O_j, O_k \rangle$ of $\mathcal{L}$}{ 
    \If{ $O_j \not \in \textnormal{Sep}(O_i,O_k)$ and both $O_i * \!\! - \!\! * O_j$ and $O_j * \!\! - \!\! * O_k$ are present in $\widehat{\mathbb{G}}$ }{
    	Orient $O_i* \!\! \linecirc O_j \circline \!\! * O_k$ as $O_i* \!\! \rightarrow O_j \leftarrow \!\! * O_k$ in $\widehat{\mathbb{G}}$ 
    }
    }

 \BlankLine

 \caption{Orienting v-structures in the RFCI algorithm}  \label{rfci_vstruc}
\end{algorithm}

\subsection{The CCD Algorithm}

The CCD algorithm considers the following problem: assume that $\mathbb{P}$ is d-separation faithful to an unknown \textit{possibly cyclic} directed graph $\mathbb{G}$. Then, given oracle information about the conditional independencies between any pair of variables $X_i$ and $X_j$ given any $\bm{W} \subseteq \bm{X}\setminus \{X_i,X_j \}$ in $\mathbb{P}$, output a partial oriented MAAG (see Section \ref{sec_MAAGs} for a definition) of the underlying directed graph \citep{Richardson96,Richardson99}. Notice that CCD does not consider latent or selection variables.

The CCD algorithm involves six steps. The first step corresponds to skeleton discovery and is analogous to PC's procedure (Algorithm \ref{pc_skel}). CCD also orients v-structures like PC. The algorithm then however checks for certain long-range d-separation relations in its third step in order to infer additional non-ancestral relations. The fourth step proceeds similarly (but not exactly) to CCI's Step \ref{alg_addD} by discovering additional non-minimal d-separating sets. Finally, the fifth and sixth steps of CCD utilize the aforementioned non-minimal d-separating sets in order to orient additional endpoints. Note that CCD does not apply orientation rules.

\section{Appendix: Proofs} \label{sec_proofs}

In the arguments to follow, I will always consider a directed graph (cyclic or acyclic) with vertices $\bm{X} = \bm{O} \cup \bm{L} \cup \bm{S}$, where $\bm{O}, \bm{L}$ and $\bm{S}$ are disjoint sets.

\subsection{Utility Lemmas}

\begin{lemma} \label{lem_d_conn} (Lemma 2.5 in Colombo et al., 2011)
Suppose that $X_i$ and $X_j$ are not in $\bm{W} \subseteq \bm{X} \setminus \{X_i, X_j\}$, there is a sequence $\sigma$ of distinct vertices in $\bm{X}$ from $X_i$ to $X_j$, and there is a set $\mathcal{T}$ of paths such that:
\begin{enumerate}
\item for each pair of adjacent vertices $X_v$ and $X_w$ in $\sigma$, there is a unique path in $\mathcal{T}$ that d-connects $X_v$ and $X_w$ given $\bm{W}$; 
\item if a vertex $X_q$ in $\sigma$ is in $\bm{W}$, then the paths in $\mathcal{T}$ that contain $X_q$ as an endpoint collide at $X_q$;
\item if for three vertices $X_v$, $X_w$ and $X_q$ occurring in that order in $\sigma$, the d-connecting paths in $\mathcal{T}$ between $X_v$ and $X_w$, and between $X_w$ and $X_q$ collide at $X_w$, then $X_w$ has a descendant in $\bm{W}$.
\end{enumerate}
Then there is a path $\Pi_{X_i X_j}$ in $\mathbb{G}$ that d-connects $X_i$ and $X_j$ given $\bm{W}$. In addition,
if all of the edges in all of the paths in $\mathcal{T}$ that contain $X_i$ are into (out of ) $X_i$, then $\Pi_{X_i X_j}$ is into (out of ) $X_i$, and similarly for $X_j$.
\end{lemma}

\begin{lemma} \label{lem_ancR}
Consider a directed graph with vertices $O_i$ and $O_j$ as well as a set of vertices $\bm{R}$ such that $O_i, O_j \not \in \bm{R}$. Suppose that there is a set $\bm{W} \setminus \{O_i, O_j\}$ such that $\bm{R} \subseteq \bm{W}$ and every proper subset $\bm{V} \subset \bm{W}$ where $\bm{R} \subseteq \bm{V}$ d-connects $O_i$ and $O_j$ given $\bm{V} \cup \bm{S}$. If $O_i$ and $O_j$ are d-separated given $\bm{W} \cup \bm{S}$ where $O_k \in \bm{W}$, then $O_k$ is an ancestor of $\{O_i, O_j\} \cup \bm{R} \cup \bm{S}$.
\end{lemma}
\begin{proof}
We will prove the claim by contrapositive. That is, we will prove the following statement: suppose that there is a set $\bm{W} \setminus \{O_i, O_j\}$ and every proper subset $\bm{V} \subset \bm{W}$ where $\bm{R} \subseteq \bm{V}$  d-connects $O_i$ and $O_j$ given $\bm{V} \cup \bm{S}$. If $O_k$ is not an ancestor of $\{O_i, O_j\} \cup \bm{R} \cup \bm{S}$, then $O_i$ and $O_j$ are d-connected given $\bm{W} \cup \bm{S}$ where $O_k \in \bm{W}$.

Let $\bm{W}^* = \textnormal{Anc}(\{O_i, O_j \} \cup \bm{R} \cup \bm{S}) \cap \bm{W}$. Note that $\bm{W}^*$ is a proper subset of $\bm{W}$ because $\bm{W}^*$ is a subset of $\bm{W} \setminus O_k$, so $O_i$ and $O_j$ must be d-connected given $\bm{W}^* \cup \bm{S}$ by a path $\Pi$ by assumption. By the definition of a d-connecting path, we know that every element in $\Pi$ must be an ancestor of $O_i$, $O_j$, $\bm{R}$, $\bm{S}$ or $\bm{W}^*$ (or some union). Moreover, because $\bm{W}^* = \textnormal{Anc}(\{O_i, O_j \} \cup \bm{R} \cup \bm{S}) \cap \bm{W}$, every element in $\bm{W}^*$ is an ancestor of $\{O_i, O_j\} \cup \bm{R} \cup \bm{S}$. Thus every element on the path $\Pi$ is an ancestor of $\{O_i, O_j\} \cup \bm{R} \cup \bm{S}$. Since $\bm{W}^* \subset \bm{W}$, the only way in which $\Pi$ could fail to d-connect $O_i$ and $O_j$ given $\bm{W} \cup \bm{S}$ would be if some element of $\bm{W} \setminus \bm{W}^*$ were located on $\Pi$. But neither $O_k$ nor any element in $\bm{W} \setminus \bm{W}^*$ is an ancestor of $\{O_i, O_j\} \cup \bm{R} \cup \bm{S}$, so it follows that no vertex in $\bm{W} \setminus \bm{W}^*$ lies on $\Pi$. We conclude that $O_i$ and $O_j$ are d-connected given $\bm{W} \cup \bm{S}$.
\end{proof}

\subsection{Step 1: Skeleton Discovery}

\begin{replemma}{lem_IP_dsep}
There exists an inducing path between $O_i$ and $O_j$ if and only if $O_i$ and $O_j$ are d-connected given $\bm{W} \cup \bm{S}$ for all possible subsets $\bm{W} \subseteq \bm{O} \setminus \{O_i, O_j \}$.
\end{replemma}
\begin{proof}
I first prove the forward direction. Consider any set $\bm{W} \subseteq \bm{O} \setminus \{O_i, O_j\}$. Suppose there exists an inducing path $\Pi$ between $O_i$ and $O_j$. We have two situations:
\begin{enumerate}
\item There exists a collider $C_1$ on $\Pi$ that is an ancestor of $O_i$ via a directed path $C_1 \leadsto O_i$ but not an ancestor of $\bm{W} \cup \bm{S}$. Let $C_1$ more specifically be such a collider on $\Pi$ closest to $O_j$. Now one of the following two conditions will hold:
\begin{enumerate}
\item There also exists a collider $C_2$ on $\Pi$ that is an ancestor of $O_j$ via a directed path $C_2 \leadsto O_j$ but not an ancestor of $\bm{W} \cup \bm{S}$. Let $C_2$ more specifically denote such a collider which is closest to $C_1$ on $\Pi$ (if two such colliders are equidistant from $C_1$, then choose one arbitrarily). Let $\Pi_{C_1 C_2}$ denote the part of the inducing path between $C_1$ and $C_2$. Recall that every non-collider on $\Pi_{C_1 C_2}$ is a member of $\bm{L}$ because $\Pi$ is an inducing path. Moreover, every collider on $\Pi_{C_1 C_2}$ is an ancestor $\bm{W}\cup \bm{S}$ by construction. Then the path $\mathcal{T} = \{ O_i \backleadsto C_1, \Pi_{C_1 C_2}, C_2 \leadsto O_j \}$ is a d-connecting path by invoking Lemma \ref{lem_d_conn} with $\mathcal{T}$.

\item There does not exist a collider $C_2$ on $\Pi$ that is an ancestor of $O_j$ via a directed path $C_2 \leadsto O_j$ and not an ancestor of $\bm{W} \cup \bm{S}$. It follows that all colliders on $\Pi$ are ancestors of $O_i \cup \bm{W} \cup \bm{S}$. More specifically, all of the colliders on $\Pi_{O_j C_1}$ are ancestors of $\bm{W} \cup \bm{S}$ by construction. Recall also that every non-collider on $\Pi_{O_j C_1}$ is a member of $\bm{L}$ because $\Pi$ is an inducing path. We conclude that the path $\mathcal{T} = \{\Pi_{O_j C_1}, C_1 \leadsto O_i \}$ is a d-connecting path by invoking Lemma \ref{lem_d_conn} with $\mathcal{T}$.

\end{enumerate}

\item There does not exist a collider $C_1$ on $\Pi$ that is an ancestor of $O_i$ via a directed path $C_1 \leadsto O_i$ and not an ancestor of $\bm{W} \cup \bm{S}$. This implies that all colliders on $\Pi$ are ancestrs of $O_j \cup \bm{W} \cup \bm{S}$. Let $\Pi_{O_i C_3}$ correspond to the part of the inducing path between $O_i$ and $C_3$, where $C_3$ corresponds to the collider closest to $O_i$ that is an ancestor of $O_j$ via a directed path $C_3 \leadsto O_j$ but not an ancestor of $\bm{W} \cup \bm{S}$; if we do not encounter such a collider, then set $C_3 = O_j$. Notice then that all colliders on $\Pi_{O_i C_3}$ are ancestors of $\bm{W} \cup \bm{S}$. Recall also that every non-collider on $\Pi_{O_i C_3}$ is a member of $\bm{L}$ because $\Pi$ is an inducing path. Thus the path $\mathcal{T} = \{ \Pi_{O_iC_3}, C_3 \leadsto O_j \}$ is a d-connecting path by invoking Lemma \ref{lem_d_conn} with $\mathcal{T}$. 

\end{enumerate}

For the backward direction, assume $O_i$ and $O_j$ are d-connecting given $\bm{W} \cup \bm{S}$ for all possible subsets $\bm{W} \subseteq \bm{O} \setminus \{O_i, O_j\}$. Then $O_i$ and $O_j$ are d-connected given $((\textnormal{Anc}(\{O_i, O_j \} \cup \bm{S}) \cap \bm{O} ) \cup \bm{S}) \setminus \{O_i, O_j\}$. The backward direction follows by invoking Lemma 8 in \citep{Spirtes99} whose argument remains unchanged even for a cyclic directed graph.
\end{proof}

\begin{replemma}{lem_dsep}
If there does not exist an inducing path between $O_i$ and $O_j$, then $O_i$ and $O_j$ are d-separated given $\textnormal{D-SEP}(O_i,O_j) \cup \bm{S}$. Likewise, $O_i$ and $O_j$ are d-separated given $\textnormal{D-SEP}(O_j,O_i) \cup \bm{S}$.
\end{replemma}
\begin{proof}
We will prove this by contradiction. Assume that we have $O_i \not \ci_d O_j|\textnormal{D-SEP}(O_i,O_j) \cup \bm{S}$. If there does not exist an inducing path between $O_i$ and $O_j$, then there exists some $\bm{W} \subseteq \bm{O} \setminus \{O_i, O_j\}$ such that $O_i \ci_d O_j|\bm{W} \cup \bm{S}$ by Lemma \ref{lem_IP_dsep}. Let $\Pi$ correspond to the path d-connecting $O_i$ and $O_j$ given $\textnormal{D-SEP}(O_i,O_j) \cup \bm{S}$.

We have two conditions:
\begin{enumerate}
\item Suppose that every vertex in $\bm{O}$ on $\Pi$ is a collider on $\Pi$. This implies that all non-colliders on $\Pi$ must be in $\bm{L} \cup \bm{S}$. But no non-collider on $\Pi$ can be in $\bm{S}$ because $\Pi$ would be inactive in that case. Thus all non-colliders on $\Pi$ must more specifically be in $\bm{L}$. Now recall that we assumed that $O_i \not \ci_d O_j|\textnormal{D-SEP}(O_i,O_j) \cup \bm{S}$, so every collider on $\Pi$ (including those in $\bm{O}$) must be an ancestor of $\textnormal{D-SEP}(O_i,O_j) \cup \bm{S}$ and hence also an ancestor of $\{O_i, O_j\} \cup \bm{S}$. The above facts imply that there exists an inducing path between $O_i$ and $O_j$; contradiction.
\item Suppose that there exists at least one vertex in $\bm{O}$ on $\Pi$ that is a non-collider. Let $O_k$ denote the first such vertex on $\Pi$ closest to $O_i$. Note that every vertex on $\Pi$ is an ancestor of $\{O_i,O_j\} \cup \textnormal{D-SEP}(O_i,O_j) \cup \bm{S}$ by the definition of d-connection and hence an ancestor of $\{O_i,O_j\} \cup \bm{S}$. This implies that $O_k$ is an ancestor of $\{O_i,O_j\} \cup \bm{S}$.

We will show that $O_k \in \textnormal{D-SEP}(O_i,O_j)$ in order to arrive at the contradiction that $\Pi$ does not d-connect $O_i$ and $O_j$ given $\textnormal{D-SEP}(O_i,O_j) \cup \bm{S}$. Consider the subpath $\Pi_{O_iO_k}$. Let $\langle C_1, \dots, C_m \rangle$ denote the possibly empty sequence of colliders on $\Pi_{O_iO_k}$ which are ancestors of $\textnormal{D-SEP}(O_i,O_j)$ but not $\bm{S}$. Also let $C_n$ denote an arbitrary collider in $\langle C_1, \dots, C_m \rangle$. Notice that there is a directed path $C_n \leadsto O_n$ with $O_n \in \textnormal{D-SEP}(O_i,O_j)$. Let $F_n$ denote the first observable on $C_n \leadsto O_n$ which may be $O_n$ if no other observable lies on $C_n \leadsto O_n$. 

We will show that there exists an inducing path between $F_n$ and $F_{n+1}$, where $F_{n+1}$ corresponds to the first observable on $C_{n+1} \leadsto O_{n+1}$. First note that $F_n, F_{n+1} \not \in \textnormal{Anc}(\bm{S})$ because $C_n,C_{n+1} \not \in \textnormal{Anc}(\bm{S})$. Consider the path $\Phi_n$ constructed by concatenating the paths $C_n \leadsto F_n$, $\Pi_{C_n C_{n+1}}$ and $C_{n+1} \leadsto F_{n+1}$. Notice that, by construction, the only observables in $\Phi_n$ lie on $\Pi_{C_n C_{n+1}}$. Moreover, every observable on $\Pi_{C_n C_{n+1}}$ is a collider because $O_k$ is the first observable that is a non-collider on $\Pi$; this implies that only a latent or a selection variable on $\Pi_{C_n C_{n+1}}$ can be a non-collider. But no selection variable is also a non-collider on $\Pi_{C_n C_{n+1}}$ because $\Pi$ d-connects $O_i$ and $O_j$ given $\textnormal{D-SEP}(O_i,O_j) \cup \bm{S}$. We conclude that only a latent variable can be a non-collider on $\Pi_{C_n C_{n+1}}$. Next, every collider on $\Pi_{C_n C_{n+1}}$ is an ancestor of $\bm{S}$ by construction of $\langle C_1, \dots, C_m \rangle$. We have shown that all colliders on $\Phi_n$ are ancestors of $\bm{S}$ and all non-colliders on $\Phi_n$ are in $\bm{L}$. This implies that $\Phi_n$ is an inducing path between $F_n$ and $F_{n+1}$; specifically one that is into $F_n$ and into $F_{n+1}$ by construction. 

We will now tie up the endpoints. We can also concatenate the paths $\Pi_{O_iC_1}$ and $C_1 \leadsto F_1$ in order to form an inducing path $\Phi_0$ between $O_i$ and $F_1$ that is into $F_1$. Similarly, we can concatenate the paths $\Pi_{O_kC_m}$ and $C_m \leadsto F_m$ in order to form an inducing path $\Phi_m$ between $O_k$ and $F_m$ that is into $F_m$.

We have constructed a sequence of vertices $\langle O_i \equiv F_0, F_1, \dots, F_m, F_{m+1} \equiv O_k \rangle$, where each vertex is an ancestor of $\{O_i,O_j\} \cup \bm{S}$ and any given $F_l$ is connected to $F_{l-1}$ by an inducing path into $F_l$ and to $F_{l+1}$ by an inducing path also into $F_l$. Hence, $O_k \in \textnormal{D-SEP}(O_i, O_j)$. But this implies that $\Pi$ does not d-connect $O_i$ and $O_j$ given $\textnormal{D-SEP}(O_i,O_j) \cup \bm{S}$ because $O_k$ is a non-collider on $\Pi$; contradiction.

\end{enumerate}

We have shown that, if there does not exist an inducing path between $O_i$ and $O_j$, then $O_i \ci_d O_j|\textnormal{D-SEP}(O_i,O_j) \cup \bm{S}$. Now $O_i \ci_d O_j|\textnormal{D-SEP}(O_i,O_j) \cup \bm{S}$ $\implies$  $O_j \ci_d O_i|\textnormal{D-SEP}(O_j,O_i) \cup \bm{S}$ because $i$ and $j$ are arbitrary indices. Moreover, $O_j \ci_d O_i|\textnormal{D-SEP}(O_j,O_i) \cup \bm{S}$ $\implies$ $O_i \ci_d O_j|\textnormal{D-SEP}(O_j,O_i) \cup \bm{S}$ because $O_j \ci_d O_i|\textnormal{D-SEP}(O_j,O_i) \cup \bm{S}$ if and only if $O_i \ci_d O_j|\textnormal{D-SEP}(O_j,O_i) \cup \bm{S}$ by symmetry of d-separation. We conclude that, if there does not exist an inducing path between $O_i$ and $O_j$, then we also have $O_i \ci_d O_j|\textnormal{D-SEP}(O_j,O_i) \cup \bm{S}$.

\end{proof}

\begin{replemma}{lem_pdsep}
If an inducing path does not exist between $O_i$ and $O_j$ in $\mathbb{G}$, then $O_i$ and $O_j$ are d-separated given $\bm{W} \cup \bm{S}$ with $\bm{W} \subseteq \textnormal{PD-SEP}(O_i)$ in the MAAG $\mathbb{G}^{\prime}$. Likewise, $O_i$ and $O_j$ are d-separated given some $\bm{W} \cup \bm{S}$ with $\bm{W} \subseteq \textnormal{PD-SEP}(O_j)$ in $\mathbb{G}^{\prime}$.
\end{replemma}
\begin{proof}
It suffices to show that $\textnormal{D-SEP}(O_i, O_j) \subseteq \textnormal{PD-SEP}(O_i)$ by Lemma \ref{lem_dsep}. The argument will hold analogously for $\textnormal{D-SEP}(O_j, O_i) \subseteq \textnormal{PD-SEP}(O_j)$. If $O_k \in \textnormal{D-SEP}(O_i, O_j)$, then there exists a sequence of observables $\Pi_{O_i, O_k}$ between $O_i$ and $O_k$ such that an inducing path exists between any two consecutive observables $\langle O_h, O_{h+1} \rangle$ in $\Pi_{O_i, O_k}$. Thus there also exists a path $\Pi^\prime_{O_i, O_k}$ between $O_i$ and $O_k$ in $\mathbb{G}^\prime$ whose vertices involve all and only the vertices in $\Pi_{O_i, O_k}$. We also know that, in every consecutive triplet $\langle O_{h-1}, O_h, O_{h+1} \rangle$, the inducing path from $O_{h-1}$ to $O_h$ is into $O_h$, and the inducing path from $O_{h+1}$ to $O_h$ is also into $O_h$; hence, $O_h$ is a collider in $\mathbb{G}$. We now need to show that 
any triplet $\langle O_{h-1}, O_h, O_{h+1} \rangle$ on $\Pi^\prime_{O_i, O_k}$ is a v-structure in $\mathbb{G}^\prime$ or a triangle in $\mathbb{G}^\prime$. We have two situations:
\begin{enumerate}
\item Suppose that the collider $O_h \not \in \textnormal{Anc}(\{O_{h-1}, O_{h+1}\} \cup \bm{S})$. Then, the path between $O_{h-1}$ and $O_h$ and then between $O_h$ and $O_{h+1}$ is not an inducing path. Hence, $O_h$ lies in an unshielded triple involving $\langle O_{h-1}, O_h, O_{h+1} \rangle$ on $\Pi^\prime_{O_i, O_k}$. If $O_h$ lies in the unshielded triple, then $O_h$ more specifically lies in a v-structure because $O_h \not \in \textnormal{Anc}(\{O_{h-1}, O_{h+1}\} \cup \bm{S})$ by assumption. 
\item Suppose that $O_h \in \textnormal{Anc}(\{O_{h-1}, O_{h+1}\} \cup \bm{S})$. Then there exists an inducing path between $O_{h-1}$ and $O_{h+1}$, so $O_h$ is in a triangle on $\Pi^\prime_{O_i, O_k}$.
\end{enumerate}
\end{proof}

\begin{replemma}{lem_pdsep2}
If an inducing path does not exist between $O_i$ and $O_j$ in $\mathbb{G}$, then $O_i$ and $O_j$ are d-separated given $\bm{W} \cup \bm{S}$ with $\bm{W} \subseteq \textnormal{PD-SEP}(O_i)$ in $\mathbb{G}^{\prime\prime}$. Likewise, $O_i$ and $O_j$ are d-separated given some $\bm{W} \cup \bm{S}$ with $\bm{W} \subseteq \textnormal{PD-SEP}(O_j)$ in $\mathbb{G}^{\prime\prime}$.
\end{replemma}
\begin{proof}
In light of Lemma \ref{lem_pdsep}, it suffices to show that $\textnormal{PD-SEP}(O_i)$ formed using the MAAG $\mathbb{G}^\prime$ is a subset of $\textnormal{PD-SEP}(O_i)$ formed using $\mathbb{G}^{\prime\prime}$. Recall that all edges in $\mathbb{G}^\prime$ are also in $\mathbb{G}^{\prime\prime}$. Hence, all triangles in $\mathbb{G}^\prime$ are also triangles in $\mathbb{G}^{\prime\prime}$. We now need to show that all v-structures in $\mathbb{G}^\prime$ are also v-structures in $\mathbb{G}^{\prime\prime}$ or are triangles in $\mathbb{G}^{\prime\prime}$. Let $\langle O_{h-1}, O_h, O_{h+1} \rangle$ denote an arbitrary v-structure in $\mathbb{G}^\prime$. The edge between $O_{h-1}$ and $O_h$ as well as the edge between $O_h$ and $O_{h+1}$ must be in $\mathbb{G}^{\prime\prime}$, because again all edges in $\mathbb{G}^\prime$ are also in $\mathbb{G}^{\prime\prime}$. We have two cases:
\begin{enumerate}
\item An edge exists between $O_{h-1}$ and $O_{h+1}$ in $\mathbb{G}^{\prime\prime}$. Then the triple $\langle O_{h-1}, O_h, $ $O_{h+1} \rangle$ forms a triangle in $\mathbb{G}^{\prime\prime}$.
\item An edge does not exist between $O_{h-1}$ and $O_{h+1}$ in $\mathbb{G}^{\prime\prime}$. Recall that $\langle O_{h-1}, O_h, O_{h+1} \rangle$ is a v-structure in $\mathbb{G}^\prime$, so $O_h \not \in \textnormal{Anc}(\{O_{h-1}, O_{h+1}\} \cup \bm{S})$. Note that PC's skeleton discovery procedure only discovers minimal separating sets so, if we have $O_{h-1} \ci_d O_{h+1} | \bm{W} \cup \bm{S}$ with $\bm{W} \subseteq \bm{O} \setminus \{O_{h-1}, O_{h+1}\}$ and $O_h \in \bm{W}$, then $O_h \in \textnormal{Anc}(\{O_{h-1}, O_{h+1}\} \cup \bm{S})$ by Lemma \ref{lem_ancR} with $\bm{R} = \emptyset$; but this contradicts the fact that $O_h \not \in \textnormal{Anc}(\{O_{h-1}, O_{h+1}\} \cup \bm{S})$. Hence $O_h \not \in \bm{W}$, so $\langle O_{h-1}, O_h, O_{h+1} \rangle$ is also a v-structure in $\mathbb{G}^{\prime\prime}$.
\end{enumerate}
\end{proof}

\subsection{Steps 3 \& 4: Short and Long Range Non-Ancestral Relations}

\begin{lemma}\label{lem_addV}
If $O_i$ is an ancestor of $O_j \cup \bm{S}$, $O_j$ and some vertex $O_k$ are d-separated given $\bm{W} \cup \bm{S}$ with $\bm{W} \subseteq \bm{O} \setminus \{O_j, O_k\}$, $O_i$ and $O_j$ are d-connected given $\bm{W} \cup \bm{S}$, and $O_i \not \in \bm{W}$, then $O_i$ and $O_k$ are d-separated given $\bm{W} \cup \bm{S}$.
\end{lemma}
\begin{proof}
Suppose for a contradiction that $O_i$ and $O_k$ are d-connected given $\bm{W} \cup \bm{S}$. There are two cases.

In the first case, suppose that $O_i$ has a descendant in $\bm{W} \cup \bm{S}$. Recall however that we have $O_i \not \in \bm{W} \cup \bm{S}$, so we can merge the d-connecting path $\Pi_{O_jO_i}$ between $O_j$ and $O_i$ and the d-connecting path $\Pi_{O_iO_k}$ between $O_i$ and $O_k$ by invoking Lemma \ref{lem_d_conn} with $\mathcal{T} = \{ \Pi_{O_jO_i}, \Pi_{O_iO_k}\}$ in order to form a d-connecting path between $O_j$ and $O_k$ given $\bm{W} \cup \bm{S}$. We have arrived at a contradiction.

In the second case, suppose that $O_i$ does not have a descendant in $\bm{W} \cup \bm{S}$. Recall also that $O_i$ is an ancestor of $O_j \cup \bm{S}$ by assumption. These two facts imply that there exists a directed path $O_i \leadsto O_j$ that does not include $\bm{W} \cup \bm{S}$; hence the $O_i \leadsto O_j$ is d-connecting. We can again invoke Lemma \ref{lem_d_conn} with $\mathcal{T} = \{ O_j \backleadsto O_i, \Pi_{O_iO_k}\}$ in order to form a d-connecting path between $O_j$ and $O_k$ given $\bm{W} \cup \bm{S}$. We have thus arrived at another contradiction.

We have exhausted all possibilities and therefore conclude that $O_i$ and $O_k$ are in fact d-separated given $\bm{W} \cup \bm{S}$.
\end{proof}

We can write the contrapositive of the above lemma as follows:
\begin{corollary} \label{cor_VS}
Let  $\bm{W} \subseteq \bm{O} \setminus \{O_j, O_k\}$. If $O_i$ and $O_j$ are d-connected given $\bm{W} \cup \bm{S}$, $O_k$ and $O_i$ are d-connected given $\bm{W} \cup \bm{S}$, $O_k$ and $O_j$ are d-separated given $\bm{W} \cup \bm{S}$, and $O_i \not \in \bm{W}$, then $O_i$ is not an ancestor of $O_j \cup \bm{S}$.
\end{corollary}

\begin{replemma}{lem_VS1}
Consider a set $\bm{W} \subseteq \bm{O} \setminus \{O_i, O_j\}$. Now suppose that $O_i$ and $O_k$ are d-connected given $\bm{W} \cup \bm{S}$, and that $O_j$ and $O_k$ are d-connected given $\bm{W} \cup \bm{S}$. If $O_i$ and $O_j$ are d-separated given $\bm{W} \cup \bm{S}$ such that $O_k \not \in \bm{W}$, then $O_k$ is not an ancestor of $\{O_i, O_j\} \cup \bm{S}$.
\end{replemma}
\begin{proof}
Follows by applying Corollary \ref{cor_VS} twice with $O_i$ and $O_k$ d-connected and with $O_j$ and $O_k$ d-connected.
\end{proof}

\subsection{Step 5: Orienting with Non-Minimal D-Separating Sets}

\begin{replemma}{lem_sup_anc2}
Consider a quadruple of vertices $\langle O_i, O_j, O_k, O_l \rangle$. Suppose that we have:
\begin{enumerate}
\item $O_i$ and $O_k$ non-adjacent.
\item $O_i * \!\! \rightarrow O_l \leftarrow \!\! * O_k$.
\item $O_i$ and $O_k$ are d-separated given some $\bm{W} \cup \bm{S}$ with $O_j \in \bm{W}$ and $\bm{W} \subseteq \bm{O} \setminus \{O_i, O_k\}$;
\item $O_j * \!\! \linecirc O_l$.
\end{enumerate}
If $O_l \not \in \bm{W} = \textnormal{Sep}(O_i,O_k)$, then we have $O_j * \!\! \rightarrow O_l$. If $O_i * \!\! \rightarrow O_j \leftarrow \!\! * O_k$ and $O_l \in \bm{W} = \textnormal{SupSep}(O_i,O_j,O_k)$, then we have $O_j * \!\! - O_l$.
\end{replemma}
\begin{proof}
We prove the first conclusion by contrapositive. Assume that we have $O_j * \!\! - O_l$. Now suppose for a contradiction that $O_l \not \in \bm{W}$ (but $O_j \in \bm{W}$). Note that $O_j \cup \bm{S}$ contains at least one descendant of $O_l$ because $O_l \in \textnormal{Anc}(O_j \cup \bm{S})$.  With Lemma \ref{lem_d_conn}, we can use the d-connecting path between $O_i$ and $O_l$ given $\bm{W} \cup \bm{S}$ as well as the d-connecting path between $O_k$ and $O_l$ given $\bm{W} \cup \bm{S}$ to form a d-connecting path between $O_i$ and $O_k$ given $\bm{W} \cup \bm{S}$ irrespective of whether or not the paths collide at $O_l$; this contradicts the fact that $O_i$ and $O_k$ are d-separated given $\bm{W} \cup \bm{S}$.

For the second conclusion, assume that we have $O_l \in \bm{W}$. We know from Lemma \ref{lem_ancR} with $\bm{R}=O_j \cup \textnormal{Sep}(O_i,O_k)$ that $O_l$ is an ancestor of $\{O_i, O_j, O_k\} \cup \textnormal{Sep}(O_i,O_k) \cup \bm{S}$. Recall that every member of $\textnormal{Sep}(O_i,O_k)$ is an ancestor of $\{O_i, O_k\} \cup \bm{S}$ by setting $\bm{R} = \emptyset$. Hence, $O_l$ is more specifically an ancestor of $\{O_i, O_j, O_k\} \cup \bm{S}$. Now since we have $O_i * \!\! \rightarrow O_l \leftarrow \!\! * O_k$, we can also claim that we have $O_l \in \textnormal{Anc}(O_j)$. Hence, we have $O_j * \!\! - O_l$.
\end{proof}

\subsection{Step 6: Long Range Ancestral Relations}

\begin{replemma}{lem_sup_anc1}
If $O_i$ and $O_k$ are d-separated given $\bm{W} \cup \bm{S}$, where $\bm{W} \subseteq \bm{O} \setminus \{ O_i, O_k\}$, and $\bm{Q} \subseteq \textnormal{Anc}(\{O_i, O_k\} \cup \bm{W} \cup \bm{S}) \setminus \{O_i, O_k \}$, then $O_i$ and $O_k$ are also d-separated given $\bm{Q} \cup \bm{W} \cup \bm{S}$.
\end{replemma}
\begin{proof}
We will prove this by contrapositive. Suppose that there is a path $\Pi_{O_iO_k}$ which d-connects $O_i$ and $O_k$ given some $\bm{Q} \cup \bm{W} \cup \bm{S}$. Then every vertex on $\Pi_{O_iO_k}$ is an ancestor of $\{O_i, O_k\} \cup \bm{Q} \cup \bm{W} \cup \bm{S}$ by the definition of a d-connecting path. Since $\bm{Q} \subseteq \textnormal{Anc}(\{O_i, O_k\} \cup \bm{W} \cup \bm{S})  \setminus \{O_i, O_k \} $, every vertex on $\Pi_{O_iO_k}$ must more specifically be an ancestor of $\{O_i, O_k\} \cup \bm{W} \cup \bm{S}$.

Let $O_a$ denote the collider furthest from $O_i$ on $\Pi_{O_iO_k}$ which is an ancestor of $O_i \cup \bm{S}$ and not in $\bm{W} \cup \bm{S}$ (or $O_i$ if no such collider exists). Similarly, let $O_b$ denote the first collider after $O_a$ on $\Pi_{O_iO_k}$ which is an ancestor of $O_k \cup \bm{S}$ and not in $\bm{W} \cup \bm{S}$ (or $O_k$ if no such collider exists). The directed path $\Pi_{O_aO_i}$ from $O_a$ to $O_i \cup \bm{S}$, and the directed path $\Pi_{O_bO_k}$ from $O_b$ to $O_k \cup \bm{S}$ are d-connecting given $\bm{W} \cup \bm{S}$, since no vertices on the path $\Pi_{O_aO_i}$ or $\Pi_{O_bO_k}$are in $\bm{W} \cup \bm{S}$. The subpath of $\Pi_{O_aO_b}$ between $O_a$ and $O_b$ on $\Pi_{O_iO_k}$ is also d-connecting given $\bm{W} \cup \bm{S}$ because every collider is an ancestor of $\bm{W} \cup \bm{S}$, and every non-collider is in $\bm{L}$. Lemma \ref{lem_d_conn} implies that we can take $\mathcal{T} = \{\Pi_{O_aO_i},\Pi_{O_aO_b},\Pi_{O_bO_k} \}$ to form a d-connecting path between $O_i$ and $O_k$ given $\bm{W} \cup \bm{S}$.
\end{proof}

\subsection{Step 7: Orientation Rules}

\begin{replemma}{lem_or_1}
Suppose that there is a set $\bm{W} \setminus \{O_i, O_j\}$ and every proper subset $\bm{V} \subset \bm{W}$ d-connects $O_i$ and $O_j$ given $\bm{V} \cup \bm{S}$. If $O_i$ and $O_j$ are d-separated given $\bm{W} \cup \bm{S}$ where $O_k \in \bm{W}$, then $O_k$ is an ancestor of $\{O_i, O_j\} \cup \bm{S}$.
\end{replemma}
\begin{proof}
This is a special case of Lemma \ref{lem_ancR} with $\bm{R} = \emptyset$.
\end{proof}

\begin{replemma}{lem_triang_main}
If we have $O_i * \!\! \rightarrow O_j \text{---} O_k$ with $O_i$ and $O_k$ non-adjacent, then $O_i * \!\! \rightarrow O_j$ is in a triangle involving $O_i,O_j$ and $O_l$ ($l \not = k$) with $O_j \text{---} O_l$ and $O_i * \!\! \rightarrow O_l$. Moreover, there exists a sequence of undirected edges between $O_l$ and $O_k$ that does not include $O_j$.
\end{replemma}
\begin{proof}
Note that $O_j$ or $O_k$ (or both) cannot be ancestors of $\bm{S}$ because this would contradict the arrowhead at $O_j$. Therefore, $O_j$ is an ancestor of $O_k$, and $O_k$ is an ancestor of $O_j$, so there is a cycle involving $O_j$ and $O_k$. Since we have an arrowhead at $O_j$, there must be an inducing path $\Pi_{O_i O_j}$ between $O_i$ and $O_j$ that is either out of $O_j$ or into $O_j$: 
\begin{enumerate}
\item Suppose that $\Pi_{O_i O_j}$ is out of $O_j$. Every vertex on $\Pi_{O_i O_j}$ is an ancestor of $\{O_i, O_j\} \cup \bm{S}$ by the definition of an inducing path. Thus, $O_j \in \textnormal{Anc}(\{O_i, O_j\} \cup \bm{S})$. Recall that we also have the arrowhead $O_i * \rightarrow O_j$, so we more specifically have the obvious relation $O_j \in \textnormal{Anc}(O_j)$. Let $C_1$ denote the collider closest to $O_j$ on $\Pi_{O_i O_j}$. Such a collider must exist or else $O_j \in \textnormal{Anc}(O_i)$ which contradicts the arrowhead $O_i * \rightarrow O_j$. Since $\Pi_{O_i O_j}$ is an inducing path, we must have $C_1 \in \textnormal{Anc}(\{O_i, O_j\} \cup \bm{S})$. However, $C_1$ cannot be an ancestor of $O_i \cup \bm{S}$ because that would imply that we have $O_j \in \textnormal{Anc}(O_i \cup \bm{S})$. We therefore more specifically have $C_1 \in \textnormal{Anc}(O_j)$. Let $C_1 \leadsto O_j$ denote a directed path to $O_j$. We have two scenarios:
\begin{enumerate}
\item $C_1 \leadsto O_j$ contains a member of $\bm{O}$ besides $O_j$. Denote that member of $\bm{O}$ closest to $C_1$ as $O_l$ (note that we may have $C_1 = O_l$). Then $\Pi_{O_i C_1}$, the part of $\Pi_{O_i O_j}$ between $O_i$ and $C_1$, as well as $C_1 \leadsto O_l$ together form an inducing path between $O_i$ and $O_l$ (every non-collider on $C_1 \leadsto O_l$ is in $\bm{L}$ by construction). Moreover, we must have $O_i * \rightarrow O_l$ because $O_l \not \in \textnormal{Anc}(O_i \cup \bm{S})$ by construction. There also exists an inducing path $\Pi_{O_j O_l}$ between $O_j$ and $O_l$ because all non-colliders on $\Pi_{O_j O_l}$ are in $\bm{L}$. We more specifically must have $O_j - O_l$ because $O_j \in \textnormal{Anc}(O_l)$ and $O_l \in \textnormal{Anc}(O_j)$ by construction. Finally, there exists a sequence of undirected edges to $O_k$ because every member of $\bm{O}$ on $C_1 \leadsto O_j$ between $O_l$ and $O_k$ is an ancestor of $O_k$ and $O_k$ is an ancestor of them.
\item $C_1 \leadsto O_j$ does not contain a member of $\bm{O}$ besides $O_j$. But then $\Pi_{O_i C_1}$ as well as $C_1 \leadsto O_j$ form an inducing path because every non-collider on $C_1 \leadsto O_j$ must be in $\bm{L}$. Hence, there exists an inducing path between $O_i$ and $O_j$ that is into $O_j$. See below for the continuation of the argument.
\end{enumerate}

\item Suppose that $\Pi_{O_i O_j}$ is into $O_j$. We also know that there is an inducing path between $O_j$ and $O_k$. Furthermore, there exists a directed path from $O_k$ to $O_j$ by the first paragraph. Hence, there exists an inducing path $\Pi_{O_j O_l}$ between some variable $O_l$ ($O_l$ is in the cycle involving $O_j$ and $O_k$ with possibly $l=k$) and $O_j$ which is into $O_j$. Suppose $l=k$; but this would imply that $O_i$ and $O_k$ are adjacent in the MAAG, since $\Pi_{O_i O_j}$ and $\Pi_{O_j O_k}$ would together form an inducing path between $O_i$ and $O_k$ (the collider $O_j$ is an ancestor of $O_k$). Hence, the inducing path must involve $O_j$ and some other observable $O_l$ where $l \not = k$. Call this inducing path $\Pi_{O_j O_l}$. Note that the path $\{\Pi_{O_i O_j}, \Pi_{O_j O_l}\}$ is an inducing path between $O_i$ and $O_l$ because $O_j$ is an ancestor of $O_l$. Thus $O_i * \!\! \rightarrow O_j$ is in a triangle involving $O_i, O_j$ and $O_l$. 

Finally recall that $O_l$ is a member of a cycle involving $O_j$ and $O_k$. Hence $O_l$ is an ancestor of $O_j$ and $O_j$ is an ancestor of $O_l$. Now $O_l$ is also not an ancestor of $\bm{S}$ because otherwise both $O_j$ and $O_k$ would also be ancestors of $\bm{S}$. Next suppose for a contradiction that $O_l$ is an ancestor of $O_i$. Then $O_j$ must be an ancestor of $O_i$ which contradicts the arrowhead $O_i * \!\! \rightarrow O_j$.
\end{enumerate}
\end{proof}

\subsection{Main Result}
\begin{reptheorem}{thm_main}
(Soundness) Consider a DAG or a linear SEM-IE with directed cyclic graph $\mathbb{G}$. If d-separation faithfulness holds, then CCI outputs a partially oriented MAAG of $\mathbb{G}$.
\end{reptheorem}
\begin{proof}
Under d-separation faithfulness, $O_i$ and $O_j$ are d-separated given $\bm{W} \subseteq \bm{O} \setminus \{O_i, O_j \}$ if and only if $O_i \ci O_j | \{ \bm{W} \cup \bm{S} \}$. Hence, we may use the terms d-separation and conditional independence as well as d-connection and conditional dependence interchangeably.

Lemma \ref{lem_IP_dsep} implies that an inducing path exists in a maximal ancestral graph if and only if $O_i$ and $O_j$ are conditionally independent given all possible subsets of $\bm{O} \setminus \{O_i, O_j\}$ as well as $\bm{S}$. Lemmas \ref{lem_dsep} and \ref{lem_pdsep} imply that we can discover the inducing paths using subsets of $\textnormal{PD-SEP}(O_i)$ and $\textnormal{PD-SEP}(O_j)$. Hence, Step \ref{alg_skeleton} of CCI is sound.

We can justify Steps \ref{alg_vstruc} and \ref{alg_addV} by invoking Lemma \ref{lem_VS1}. Correctness of Step \ref{alg_step5} follows by by Lemma \ref{lem_sup_anc2}, and Step \ref{alg_step6} by the contrapositive of Lemma \ref{lem_sup_anc1}. Finally, correctness of the orientation rules follows by invoking Lemmas \ref{lem_OR123}, \ref{lem_OR45} and \ref{lem_OR67} for orientation rules 1-3, 4-5 and 6-7, respectively. 
\end{proof}

\clearpage

% unstr is used to keep citation order
%\bibliography{htp-amia}  

\end{document}